\documentclass{article}

\usepackage[numbers, sort&compress]{natbib}

\usepackage[utf8]{inputenc} %
\usepackage[T1]{fontenc}    %
\usepackage{hyperref}       %
\usepackage{url}            %
\usepackage{booktabs}       %
\usepackage{amsfonts}       %
\usepackage{nicefrac}       %
\usepackage{microtype}      %
\usepackage{subfigure}
\usepackage{latexsym}
\usepackage{amsmath}
\usepackage{appendix}
\usepackage{amssymb}
\usepackage{wrapfig}
\usepackage[margin=1.45in]{geometry}
\usepackage{shortcuts}

\usepackage{amsmath, amsthm, amssymb} 
\usepackage{bm}
\usepackage{natbib}
\usepackage{tikz}
\usepackage{comment}
\usepackage{enumerate,xcolor}   
\DeclareMathOperator*{\argmax}{arg\,max}
\DeclareMathOperator*{\argmin}{arg\,min}
\def\che{\tikz\fill[scale=0.4](0,.35) -- (.25,0) -- (1,.7) -- (.25,.15) -- cycle;}

\newtheorem{theorem}{Theorem}[section]
\newtheorem{assumption}{Assumption}[section]
\newtheorem{corollary}{Corollary}[section]
\newtheorem{definition}{Definition}[section]

\newtheorem{remark}{Remark}[section]

\newcommand{\bx}{\mathbf{x}}
\newcommand{\ba}{\mathbf{a}}

\newcommand{\op}{\mathrm{o}_{p}}

\usepackage{color,graphicx,lscape,longtable}

\newtheorem{lemma}{Lemma}[section]

\title{Intrinsically Efficient, Stable, and Bounded\\Off-Policy Evaluation for Reinforcement Learning}

\author{%
Nathan Kallus \\
   Cornell University and Cornell Tech \\
   New York, NY \\
   \texttt{kallus@cornell.edu} \\
   \and
  Masatoshi Uehara\footnote{corresponding author}\\
  Harvard University\\
    Cambrdige, MA\\
  \texttt{uehara\_m@g.harvard.edu} 
}
\date{}
\begin{document}

\maketitle

\begin{abstract}
Off-policy evaluation (OPE) in both contextual bandits and reinforcement learning allows one to evaluate novel decision policies without needing to conduct exploration, which is often costly or otherwise infeasible. The problem's importance has attracted many proposed solutions, including importance sampling (IS), self-normalized IS (SNIS), and doubly robust (DR) estimates. DR and its variants ensure semiparametric local efficiency if Q-functions are well-specified, but if they are not they can be worse than both IS and SNIS. It also does not enjoy SNIS's inherent stability and boundedness. We propose new estimators for OPE based on empirical likelihood that are always more efficient than IS, SNIS, and DR and satisfy the same stability and boundedness properties as SNIS. On the way, we categorize various properties and classify existing estimators by them. Besides the theoretical guarantees, empirical studies suggest the new estimators provide advantages.
\end{abstract}

\section{Introduction}

Off-policy evaluation (OPE) is the problem of evaluating a given policy (evaluation policy) using data generated by the log of another policy (behavior policy). OPE is a key problem in both reinforcement learning (RL) \citep{Precup2000,Mahmood20014,Li2015,thomas2016,jiang,Munos2016,Liu2018} and contextual bandits (CB) \citep{DudikMiroslav2014DRPE,Swaminathan2015b,narita2018} and it finds applications as varied as healthcare \citep{MurphyS.A.2003Odtr} and education \citep{Mandel2014}. 

Methods for OPE can be roughly categorized into three types. The first approach is the \emph{direct method} (DM), wherein we directly estimate the Q-function using regression and use it to directly estimate the value of the evaluation policy.
The problem of this approach is that if the model is wrong (misspecified), the estimator is no longer consistent. 

The second approach is importance sampling (IS; aka Horvitz-Thompson), which averages the data weighted by the density ratio of the evaluation and behavior policies. Although IS gives an unbiased and consistent estimate, its variance tends to be large. 
Therefore self-normalized IS (SNIS; aka H\'ajek) is often used \citep{Swaminathan2015b}, which divides IS by the average of density ratios.
SNIS has two important properties: (1) its value is bounded in the support of rewards and (2) its conditional variance given action and state is bounded by the conditional variance of the rewards.
This leads to increased stability compared with IS, especially when the density ratios are highly variable due to low overlap.

The third approach is the doubly robust (DR) method, which combines DM and IS and is given by adding the estimated Q-function as a control variate \citep{RobinsJamesM.1994EoRC,DudikMiroslav2014DRPE,jiang}. If the Q-function is well specified, DR is locally efficient 
in the sense that its asymptotic MSE 
achieves the semiparametric lower bound 
\citep{tsi}. However, if the Q-function is misspecified, DR can actually have \emph{worse} MSE than IS and/or SNIS \citep{KangJosephD.Y.2007DDRA}.
In addition, it does not have the boundedness property.

To address these deficiencies, we propose novel OPE estimators for both CB and RL that are guaranteed to improve over both (SN)IS and DR in terms of asymptotic MSE (termed \emph{intrinsic efficiency}) and at the same time also satisfy the same boundedness and stability properties of SNIS, in addition 
to the consistency and local efficiency of existing DR methods. See Table \ref{tab:comparison}. 
Our general strategy to obtain these estimators is to (1) make a parametrized class of estimators that includes IS, SNIS, and DR and (2) choose the parameter using either a regression way (REG) or an empirical likelihood way (EMP). 
The benefit of these new properties in practice is confirmed by experiments in both CB and RL settings.

\begin{table}[!]%
\centering%
\caption{Comparison of policy evaluation methods. The notation (*) means proposed estimator. The notation \# means partially satisfied, as discussed in the text. (S)IS and SN(S)IS refer either to stepwise or non-stepwise.}%
\setlength{\tabcolsep}{.325em}%
\begin{tabular}{lccccccccc}%
& DM & (S)IS & SN(S)IS & DR &  SNDR &MDR  & REG(*) & SNREG(*) & EMP(*) \\
Consistency  & & \che  & \che & \che & \che  & \che  & \che &\che  & \che\\
Local efficiency & \che  &  &  &\che & \che  &\che  & \che & \che & \che \\
Intrinsic efficiency &  &  &   &  &   & \# & \che & \# & \che\\
Boundedness &  1 &   & 1  &  & 2 &  &  &  2 &  1 \\
Stability &  \# &   & \che  &  & \# &  &  &   &  \che \\
\end{tabular}\label{tab:comparison}%
\end{table}

\section{Markov Decision Processes and Off Policy Evaluation}

An MDP is defined by a tuple $(\mathcal{X},\mathcal{A},P,R,P_0,\gamma)$, where $S$ and $A$ are the state and action spaces, $P_{r}(x,a)$ is the distribution of the bounded random variable $r(x,a)\in [0,R_{\mathrm{max}}]$ being the immediate reward of taking action $a$ in state $x$, $P(\cdot|x,a)$ is the transition probability distribution, $P_{0}$ is the initial state distribution, and $\gamma\in[0,1]$ is the discounting factor. A policy $\pi:\mathcal{X}\times \mathcal{A}\to [0,1]$ assigns each state $x\in \mathcal{X}$ a distribution over actions with $\pi(a|x)$ being the probability of taking actions $a$ into $x$. We denote  $\mathcal{H}_{T-1}=(x_0,a_0,r_0,\cdots,x_{T-1},a_{T-1},r_{T-1})$ as a T-step trajectory generated by policy $\pi$, and define $R_{T-1}(\mathcal{H}_{T-1})=\sum_{t=0}^{T-1}\gamma^{t}r_t$, which is the return of trajectory. Our task is to estimate 
$$\beta^{\pi}_{T}=\mathrm{E}[R_{T-1}(\mathcal{H}_{T-1})]\qquad\qquad\text{(policy value)}.$$
We further define the value function $V^{\pi}(x)$ and Q-function $Q^{\pi}(x,a)$ of a policy $\pi$, respectively, as the expectation of the return of a $T$-step trajectory generated by starting at state $x$ and state-action pair $(x,a)$. Note that the contextual bandit setting is a special case when $T=1$.

The off-policy evaluation (OPE) problem is to estimate $\beta^{*}=\beta^{\pi_e}_{T}$ for the evaluation policy $\pi_e$ from $n$ observation of $T$-step trajectories $\mathcal{D}=\{\mathcal{H}_{T-1}^{(i)}\}_{i=1}^{n}$ independently generated by the behavior policy $\pi_b$. 
Here, we assume an overlap condition: for all state-action pair $(x,a)\in \mathcal{X}\times \mathcal{A}$ if $\pi_{b}(a|x)=0$ then $\pi_{e}(a|x)=0$. 
Throughout, expectations $\mathrm{E}[\cdot]$ are taken with respect to a behavior policy. 
For any function of the trajectory, we let $$\textstyle\mathrm{E}_{n}[f(\mathcal H_{T-1})]=n^{-1}\sum_{i=1}^nf(\mathcal H_{T-1}^{(i)}).$$
$\mathrm{Asmse}[\cdot]$ denotes asymptotic MSE in terms of the first order; \ie, $\mathrm{Asmse}[\hat\beta]=\mathrm{MSE}[\hat\beta]+\mathrm{o}(n^{-1})$.

The cumulative importance ratio from time step $t_1$ to time step $t_2$ is
$$\textstyle \omega_{t_1:t_2}=\prod_{t=t_1}^{t_2}{\pi_e(a_t|x_t)}/{\pi_b(a_t|x_t)},$$
where the empty product is $1$.
We assume that this weight is bounded for simplicity. 

\subsection{Existing Estimators and Properties}

We summarize three types of estimators. Some estimators depend on a model $q(x,a;\tau)$ with parameter $\tau\in\Theta_\tau$ for the Q-function $Q^{\pi_e}(x,a)$. We say the model is correct or well-specified if there is some $\tau_0$ such that $Q^{\pi_e}(x,a)=q(x,a;\tau_0)$ and otherwise we say it is wrong or misspecified.
Throughout, we make the following assumption about the model
\begin{assumption}
(a1) $\Theta_\tau$ is compact,
(a2) $\abs{q(x,a;\tau)}\leq R_{\max}$.
\end{assumption}

\textbf{Direct estimator:}\, 
DM is given by fitting $\hat{\tau}$, \eg, by least squares, and then plugging this into
\begin{align}
\label{eq:dm}
    \hat{\beta}_{\mathrm{dm}}=\E_n\bracks{\sum_{a\in \mathcal{A}}\pi_e(a|x^{(i)}_{0})q(x_0^{(i)},a;\hat{\tau})}.
\end{align}
When this model is correct, $\hat{\beta}_{\mathrm{dm}}$ is both consistent for $\beta^*$ and locally efficient in that its asymptotic MSE is minimized among the class of all estimators consistent for $\beta^*$
\citep{narita2018,tsi}. 

\begin{definition}[Local efficiency]
\label{def:local}
When the model $q(x,a;\tau)$ is well-specified, the estimator achieves the efficiency bound. 
\end{definition}
However, all of models are wrong to some extent. In this sense, even if the sample size goes to infinity, $\hat{\beta}_{\mathrm{dm}}$ might not be consistent. 
\begin{definition}[Consistency] The estimator is consistent for $\beta^*$ irrespective of model specification.
\end{definition}

\begin{wrapfigure}{r}{35mm}\vspace{-2em}\centering%
\subfigure[Well-specified]{\centering\begin{tikzpicture}[node distance=0.5cm]%
\node[draw=none,fill=none](a){$\operatorname{EMP}=\operatorname{REG}=\operatorname{DR}$};
\node[draw=none,fill=none,below of=a,rotate=90](b){$\Big>$};
\node[draw=none,fill=none,below of=b](c){$\operatorname{IS},\,\operatorname{SNIS}$};
\end{tikzpicture}}
\subfigure[Misspecified]{\centering\begin{tikzpicture}[node distance=0.5cm]%
\node[draw=none,fill=none](a){$\operatorname{EMP}=\operatorname{REG}$};
\node[draw=none,fill=none,below of=a,rotate=90](b){$\Big>$};
\node[draw=none,fill=none,below of=b](c){$\operatorname{IS},\,\operatorname{SNIS},\,\operatorname{DR}$};
\end{tikzpicture}}
\vspace{-0.33em}\caption{Order of asymptotic MSEs}\label{fig:order}\vspace{-3em}%
\end{wrapfigure}
\textbf{Importance sampling estimators:} Importance sampling (IS) and step-wise importance sampling (SIS) are defined respectively as 
\begin{align*}
    \hat{\beta}_{\operatorname{is}}=\mathrm{E}_{n}\left[\omega_{0:T-1}\sum_{t=0}^{T-1}\gamma^{t}r_{t}\right],\ \hat{\beta}_{\operatorname{sis}}=\mathrm{E}_{n}\left[\sum_{t=0}^{T-1}\omega_{0:t}\gamma^{t}r_{t}\right]. 
\end{align*}
Both satisfy consistency but the MSE of SIS estimator is smaller than regular IS estimator by the law of total variance \citep{SuttonRichardS1998Rl:a}. 

The self-normalized versions of these estimators are: 
\begin{align*}
\hat{\beta}_{\operatorname{snis}}=\frac{\mathrm{E}_{n}\left[{\omega_{0:T-1}}\sum_{t=0}^{T-1}\gamma^{t}r_{t}\right]}{\mathrm{E}_{n}[\omega_{0:T-1}]},\ \hat{\beta}_{\operatorname{snsis}}=\mathrm{E}_{n}\left[\sum_{t=0}^{T-1}\frac{\omega_{0:t}}{\mathrm{E}_{n}[\omega_{0:t}]}\gamma^{t}r_{t}\right].
\end{align*}
SN(S)IS have two advantages over (S)IS. 
First, they are both 1-bounded in that they are bounded by the theoretical upper bound of reward.
\begin{definition}[$\alpha$-Boundedness]
The estimator is bounded by $\alpha\sum_{t=0}^{T-1}\gamma^{t}R_{\mathrm{max}}$.
\end{definition}
1-boundedness is the best we can achieve where $\alpha$-boundedness for any $\alpha>1$ is a weaker property.
Second, their conditional variance given state and action data are no larger than the conditional variance of any reward, to which we refer as \emph{stability}.
\begin{definition}[Stability]
Let $\mathcal D_{x,a}=\{(x\s i_t,a\s i_t):i\leq n,t\leq T-1\}$ denote that action-state data.
If the conditional variance of $\sum_{t=1}^{T-1}\gamma^tr\s i_t$, given $\mathcal D_{x,a}$, is bounded by $\sigma^2$, then the conditional variance of the estimator, given $\mathcal D_{x,a}$, is also bounded by $\sigma^2$.
\end{definition}
Unlike efficiency, boundedness and stability are finite-sample properties.
Notably (S)IS lacks both of these properties, which explains its unstable performance in practice, especially when density ratios can be very large. While boundedness can be achieved by a simple truncation, stability cannot.

\textbf{Doubly robust estimators:} A DR estimator for RL \citep{jiang,thomas2016} is given by fitting $\hat\tau$ and plugging it into $$\hat{\beta}_{\operatorname{dr}}=\hat{\beta}_{d}(\{q(x,a;\hat{\tau})\}_{t=0}^{T-1}),$$ where for any collection of functions $\{m_t\}_{t=0}^{T-1}$ (known as control variates) we let %
\begin{align}
\label{eq:drclass_gene}
 \hat{\beta}_{d}(\{m_t\}_{t=0}^{T-1})=\mathrm{E}_{n}\left[\sum_{t=0}^{T-1}\gamma^{t}\omega_{0:t}r_{t}-\gamma^{t}\left(\omega_{0:t}m_{t}(x_t,a_t)-\omega_{0:t-1}\left \{\sum_{a \in A} m_{t}(x_t,a){\pi_{e}}(a|x_t)\right \} \right) \right].
\end{align}
The DR estimator is both consistent and locally efficient. 
Instead of using a plug-in estimate of $\tau$,
\citep{CaoWeihua2009Iear,rubin08,Chow2018} further suggest that to pick $\hat\tau$ to minimize an estimate of the asymptotic variance of $\hat{\beta}_{d}(\{q(x,a;{\tau})\}_{t=0}^{T-1})$, leading to the MDR estimator \citep{Chow2018} for OPE.
However, DR and MDR satisfy neither boundedness nor stability. Replacing, $\omega_{0:t}$ with its self-normalized version $\omega_{0:t}/\E_n\bracks{\omega_{0:t}}$ in \eqref{eq:drclass_gene} leads to SNDR \cite{RobinsJames2007CPoD,thomas2016} (aka WDR), but it only satisfies these properties partially: it's only 2-bounded and partially stable (see Appendix \ref{sec:bdrs}).

Moreover, if the model is incorrectly specified, (M)DR may have MSE that is worse than any of the four (SN)(S)IS estimators.
\cite{KangJosephD.Y.2007DDRA} also experimentally showed that the performance of $\hat{\beta}_{\operatorname{dr}}$ might be very bad in practice when the model is wrong.

We therefore define \emph{intrinsic efficiency} as an additional desiderata, which prohibits this from occurring.
\begin{definition}[Intrinsic efficiency] 
\label{def:improve}
The asymptotic MSE of the estimator is smaller than that of any of $\hat{\beta}_{\operatorname{sis}},\hat{\beta}_{\operatorname{is}},\hat{\beta}_{\operatorname{snsis}},\hat{\beta}_{\operatorname{snis}},\hat{\beta}_{\operatorname{dr}}$, irrespective of model specification.
\end{definition}

MDR can be seen as motivated by a variant of intrinsic efficiency against only DR (hence the \# in Table~\ref{tab:comparison}). Although this is not precisely proven in \cite{Chow2018}, this arises as a corollary of our results. Nonetheless, MDR does not achieve full intrinsic efficiency against all above estimators.

\section{REG and EMP for Contextual Bandit}

None of the estimators above simultaneously satisfy all desired properties, Definitions \ref{def:local}--\ref{def:improve}. In the next sections, we develop new estimators that do. For clarity we first consider the simpler CB setting, where we write $(x,a,r)$ and $w$ instead of $(x_0,a_0,r_0)$ and $w_{0:0}$. We then start by showing how a modification to MDR ensures intrinsic efficiency. To obtain the other desiderata, we have to change how we choose the parameters.

\subsection{REG: Intrinsic Efficiency}\label{sec:regcb}

When $T=1$, $\hat{\beta}_{d}(m)$ in \eqref{eq:drclass_gene} becomes simply
\begin{align}
\label{eq:drclass}
\hat{\beta}_{d}(m)=\E_n\bracks{wr-\mathcal{F}(m)},
\end{align}
where $\mathcal{F}(m(x,a))=w m(x,a)-\left\{\sum_{a \in A}m(x,a)\pi_{e}(a|x)\right \}$.
By construction, $\mathrm{E}[\mathcal{F}(m)]=0$ for every $m$. (M)DR, for example, use $m(x,a;\tau)=q(x,a;\tau)$.
Instead, we let $$m(x,a;\zeta_1,\zeta_2,\tau)=\zeta_1+ \zeta_2 q(x,a;\tau),$$ for parameters $\tau$ and $\zeta=(\zeta_1,\zeta_2)$.
This new choice has a special property: it includes both IS and DR estimators. 
Given any $\tau$, setting $\zeta_1=0,\zeta_2=0$ yields IS and setting $\zeta_1=0,\zeta_2=1$ gives (M)DR. 
This gives a simple recipe for intrinsic efficiency: estimate the variance of $\hat{\beta}_{d}(\zeta_1+\zeta_2q(x,a;\tau))$ and minimize it over $\tau,\zeta$.
Because $\hat{\beta}_{d}(m)$ is unbiased,
its variance is simply $\mathrm{E}\left[\{wr-\mathcal{F}(m)\}^{2} \right]-{\beta^{*}}^{2}$. Therefore, over the parameter spaces $\Theta_\tau$ the (unknown) minimal variance choice is
\begin{align}
\label{eq:optimzation}
    (\zeta^{*},\tau^{*})=\argmin_{\zeta \in \R2,\tau \in \Theta_\tau} \mathrm{E}\left[\left\{wr- \mathcal{F}(\zeta_1+ \zeta_2 q(x,a;\tau))\right\}^{2}\right].
\end{align}
We let the REG estimator be $\hat{\beta}_{\operatorname{reg}}= \hat{\beta}_{d}(\hat{\zeta}_1+\hat{\zeta}_2 q(x,a;\hat{\tau}))$ where we choose the parameters by minimizing the estimated variance:
\begin{align}
\label{eq:optimzation2}
(\hat{\zeta},\hat{\tau})=\argmin_{\zeta \in \R2,\tau \in \Theta_\tau} \mathrm{E}_{n}\left[\left\{wr- \mathcal{F}(\zeta_1+ \zeta_2 q(x,a;\tau))\right\}^{2}\right].
\end{align}

To establish desired efficiencies, we prove the following theorem indicating that our choice of parameters does not inflate the variance.
Note that it is not obvious because the plug-in some parameters generally causes an inflation of the variance. 

\begin{theorem}
\label{thm:drs}
When the optimal solution $(\zeta^{*},\tau^{*})$ in \eqref{eq:optimzation} is unique, 
\begin{align*}
    \mathrm{Asmse}[\hat{\beta}_{\operatorname{reg}}]=n^{-1}\min_{\zeta \in \R2,\tau \in \Theta_\tau}\mathrm{E}\left[\left\{wr- \mathcal{F}(\zeta_1+\zeta_2 q(x,a;\tau))\right\}^{2}-{\beta^{*}}^{2}\right].
\end{align*}
\end{theorem}
\begin{remark}
\label{rem:bis}
For $\zeta=(\beta^{*},0)$, this asymptotic MSE is the same as the one of SNIS, $\mathrm{var}[w(r-\beta^{*})] $. 
\end{remark}

From Theorem \ref{thm:drs} we obtain the desired efficiencies. Importantly, to prove this, we note how the asymptotic MSEs of each of (SN)(S)IS and DR can be represented in the form $n^{-1}\mathrm{E}\left[\left\{wr- \mathcal{F}(\zeta_1+\zeta_2 q(x,a;\tau))\right\}^{2}-{\beta^{*}}^{2}\right]$ for some $\zeta$ and $\tau$.
\begin{corollary}
\label{col:drs_imp}
The estimator $\hat{\beta}_{\operatorname{reg}}$ has local and intrinsic efficiency.
\end{corollary}
 
\begin{remark}[Comparison to MDR]
REG is like MDR with an expanded model class. This class is carefully chosen to guarantee intrinsic efficiency. 
In addition, as another corollary, we have proven partial intrinsic efficiency for MDR against DR (just fix $\zeta=(0,1)$ in \eqref{eq:optimzation2}) where \cite{Chow2018} only proved consistency of MDR.
However, neither MDR nor REG satisfies boundedness and stability.
\end{remark}

\begin{remark}[SNREG]
Replacing weights $w$ by their self-normalized version $w/\E_n[w]$ in REG leads to SNREG. We explore this estimator in Appendix A and show it only gives 2-boundedness, does not give stability, and limits REG's intrinsic efficiency to be only against SN(S)IS and SNDR.
\end{remark}

\subsection{EMP: Intrinsic Efficiency, Boundedness, and Stability}

We next construct an estimator satisfying intrinsic efficiency as well as boundedness and stability. The key idea is to use empirical likelihood to choose the parameters \citep{TanZhiqiang2004OLAf,TanZhiqiang2006ADAf,TanZhiqiang2010Bead}. Empirical likelihood is a nonparametric MLE commonly used in statistics \citep{OwenArt.ArtB.2001El}. We consider the control variate $m(x,a;\xi;\tau)=\xi + q(x,a;\tau)$ with parameters $\xi,\tau$ and $q(x,a;\tau)=t(x,a)^{\top}\tau$, where $t(x,a)$ is a $d_\tau$-dimensional vector of linear independent basis functions not including a constant. Then, an estimator for $\beta$ is defined as 
\begin{align}\notag
\hat{\beta}_{\operatorname{emp}} &= \mathrm{E}_{n}\left[\hat{c}^{-1}\hat{\kappa}(x,a)\pi_{e}(a|x)r\right ],\quad \text{where}\\
\hat{\kappa}(x,a)&= \{\pi_{b}(a|x)[1+\mathcal{F}(m(x,a;\hat{\xi },\hat{\tau}))] \}^{-1},\ 
\hat{c} = \mathrm{E}_{n}\left[\{1+\mathcal{F}(m(x,a;\hat{\xi },\hat{\tau}))\}^{-1}\right],\notag\\
\label{eq:drss_optim}
    \hat{\xi}, \hat{\tau} &=\argmax_{\xi \in \mathbb{R}, \tau \in \Theta_\tau}\mathrm{E}_{n}[\log \{1+\mathcal{F}(m(x,a;\xi,\tau))\}]. 
\end{align}
This is motivated by solving the dual problem of the following optimization problem formulated by the empirical likelihood:
\begin{align*}
    \max_{\kappa} \sum_{i=1}^{n}\log \kappa^{(i)},\,\mathrm{s.t.} \sum_{i=1}^{n}\kappa^{(i)}\pi_{b}(a^{(i)}|x^{(i)})=1,\,\sum_{i=1}^{n}\kappa^{(i)}\pi_{b}(a^{(i)}|x^{(i)})\mathcal{F}(m(x^{(i)},a^{(i)};\xi,\tau))=0. 
\end{align*}

The objective in an optimization problem \eqref{eq:drss_optim} is a convex function; therefore, it is easy to solve. Then, the estimator $\hat{\beta}_{\operatorname{emp}}$ has all the desirable finite-sample and asymptotic properties. 

\begin{lemma}
\label{lem:drss}
The estimator $\hat{\beta}_{\operatorname{emp}}$ satisfies 1-boundedness and stability. 
\end{lemma}

\begin{theorem}
\label{thm:drss}
The estimator $\hat{\beta}_{\operatorname{emp}}$ has local and intrinsic efficiency, and  
\begin{align}
\label{eq:opt-linear}
\mathrm{Asmse}[\hat{\beta}_{\operatorname{emp}}]=
n^{-1} \min_{\zeta \in \mathbb{R},\tau \in \mathbb{R}^{d_\tau}}\mathrm{E}\left[\left \{wr-\mathcal{F}(\zeta+ q(x,a;\tau))\right\}^{2}-{\beta^{*}}^{2}\right].  
\end{align}
\end{theorem}
Here, we have assumed the model is linear in $\tau$. 
Without this assumption, Theorem \ref{thm:drss} may not hold. In the following section, we consider how to relax this assumption while 
maintaining local and intrinsic efficiency. 

\subsection{Practical REG and EMP}
\label{sec:pra}

While REG and EMP have desirable theoretical properties,
both have some practical issues. First, for REG, the optimization problem in \eqref{eq:optimzation2} may be non-convex if $q(x,a;\tau)$ is not linear in $\tau$, as is the case in our experiment in Sec.~\ref{sec:expcb} where we use a logistic model with $216$ parameters.
(The same issue exists for MDR.) Similarly, EMP estimator has the problem that there is no theoretical guarantee for intrinsic efficiency when $q(x,a;\tau)$ is not linear in $\tau$. Therefore, we suggest the following unified practical approach to selecting $\tau$ in a way that maintains the desired properties.

First, we estimate a parameter $\tau$ in $q(x,a;\tau)$ as in DM to obtain $\hat{\tau}$, which we assume as a limit, $\hat{\tau}\stackrel{p}{\rightarrow}\tau^{\dagger}$
. Then, we consider solving the following optimization problems instead of \eqref{eq:optimzation2} and \eqref{eq:drss_optim} for REG and EMP, respectively
\begin{align*}
    \hat{\zeta}=\argmin_{\zeta \in \mathbb{R}^{2}}\mathrm{E}_{n}\left[\left\{wr- \mathcal{F}(m(x,a;\zeta,\hat{\tau}))\right\}^{2}\right],\,
    \hat{\xi} =\argmax_{\xi \in \mathbb{R}^{2}}\mathrm{E}_{n}[\log \{1+\mathcal{F}(m(x,a;\xi,\hat{\tau}))\}],
\end{align*}
where $m(x,a;\zeta,\hat{\tau})=\zeta_1+\zeta_2 q(x,a;\hat{\tau})$ or $m(x,a;\xi,\hat{\tau})=\xi_1+\xi_2 q(x,a;\hat{\tau})$. This is a convex optimization problem with two dimensional parameters; thus, it is easy to solve. 

Here, the asymptotic MSE of practical $\hat{\beta}_{\operatorname{reg}}$ and $\hat{\beta}_{\operatorname{emp}}$ are as follows. 
\begin{theorem}
\label{thm:pra_drs}
The above plug-in-$\tau$ versions of $\hat{\beta}_{\operatorname{reg}}$ and $\hat{\beta}_{\operatorname{emp}}$ still satisfy local and intrinsic efficiency, and $\hat{\beta}_{\operatorname{emp}}$ satisfies 1-boundedness and partial stability. Their asymptotic MSEs are 
\begin{align}
\label{eq:prac_drs}
   n^{-1}\min_{\zeta \in \mathbb{R}^{2}}\mathrm{E}\left[\left\{wr- \mathcal{F}(\zeta_1+\zeta_2 q(x,a;\tau^{\dagger}))\right\}^{2}-{\beta^{*}}^{2}\right].
\end{align}

\end{theorem}

As a simple extension, we may consider multiple models for the Q-function. E.g, we can have two models $q_1(x,a;\tau_1)$ and $q_2(x,a;\tau_1)$ and let $m(x,a;\zeta,\hat{\tau})=\zeta_1+\zeta_2q_1(x,a;\hat\tau_1)+\zeta_3q_2(x,a;\hat\tau_2)$. Our results easily extend to provide intrinsic efficiency with respect to DR using any of these models.

\section{REG and EMP for Reinforcement learning}

We next present how REG and EMP extend to the RL setting. Some complications arise because of the multi-step horizon. For example, IS and SIS are different as opposed to the case $T=1$. 

\subsection{REG for RL}

We consider an extension of REG to a RL setting. First, we derive the variance of $\hat{\beta}_{d}(\{m_t\}_{t=0}^{T-1})$. 
\begin{theorem}
\label{thm:rdrs}
The variance of $\hat{\beta}_{d}(\{m_t\}_{t=0}^{T-1})$ is 
$n^{-1}\mathrm{E}[v(\{m_t\}_{t=0}^{T-1})]$, where $v(\{m_t\}_{t=0}^{T-1})$ is 
\begin{align}
\label{eq:var_multi}
\sum_{t=0}^{T-1}\gamma^{2t}\omega_{0:t-1}^{2}\mathrm{var}\left(\mathrm{E}[\sum_{k=t}^{T-1}\gamma^{k-t}\omega_{t:k}r_{k-t}|\mathcal{H}_{t}]-\left \{\omega_{t:t}m_t(x_t,a_t)-\sum_{a\in A} m_t(x_t,a)\pi_{e}(a|x_t) \right \} |\mathcal{H}_{t-1}\right).
\end{align}
\end{theorem}
To derive REG, we consider the class of estimators
$\hat{\beta}_{d}(\{m_t\}_{t=0}^{T-1})$
where $m_t$ is $m_t(x_t,a_t;\zeta)=\zeta_{1t}+\zeta_{2t} q(x_t,a_t;\hat{\tau})$ for all $0 \leq t \leq T-1$. Then, we define an estimator $\hat{\zeta}$  and the optimal $\zeta^*$ as 
\begin{align}
\label{eq:rl-drs}
    \hat{\zeta} =\argmin_{\zeta \in \mathbb{R}^{2}}\mathrm{E}_{n}[v(\{m_t(x_t,a_t;\zeta)\}_{t=0}^{T-1})],\quad \zeta^{*} =\argmin_{\zeta \in \mathbb{R}^{2}}\mathrm{E}[v(\{m_t(x_t,a_t;\zeta)\}_{t=0}^{T-1})].
\end{align}
REG is then defined as $\hat{\beta}^{T-1}_{\operatorname{reg}}=\hat{\beta}_{d}(\{\hat{\zeta}_{1t}+\hat{\zeta}_{2t} q(x,a;\hat{\tau})\}_{t=0}^{T-1})$, where following our discussion in Section~\ref{sec:pra}, $\hat\tau$ is given by fitting as in DM/DR. Theoretically, we could also choose $\tau$ to minimize eq.~\eqref{eq:var_multi}, but that can be computationally intractable.

A similar argument to that in Section~\ref{sec:regcb} shows that a data-driven parameter choice induces no inflation in asymptotic MSE. Therefore, the asymptotic MSE of the estimator $\hat{\beta}_{\operatorname{reg}}$ is minimized among the class of estimators $\hat{\beta}_{d}(\{\zeta_{1t}+\zeta_{2t} q(x_t,a_t;\hat{\tau})\}_{t=0}^{T-1}))$. This implies that the asymptotic MSE of $\hat{\beta}_{\operatorname{reg}}$ is smaller than $\hat{\beta}_{\operatorname{sis}}$ and $\hat{\beta}_{\operatorname{dr}}$ because $\hat{\beta}_{\operatorname{sis}}$ corresponds to the case $\zeta_t=(0,0)$ and $\hat{\beta}_{\operatorname{dr}}$ corresponds to the case $\zeta_t=(0,1)$. In addition, we can prove that the estimator $\hat{\beta}^{T-1}_{\operatorname{reg}}$ is more efficient than $\hat{\beta}_{\operatorname{snsis}}$. To prove this, we introduce the following lemma. 

\begin{lemma}
\label{lem:rbdrs}
\begin{align*}
\mathrm{Asmse}[\hat{\beta}_{\operatorname{snis}}]=n^{-1}
\sum_{t=0}^{T-1}\mathrm{E}\left[\gamma^{2t}\omega_{0:t-1}^{2}\mathrm{var}\left(\omega_{t:t}\left(\mathrm{E}\left[\sum_{k=t}^{T-1}\gamma^{k-t}\omega_{t+1:k}r_{k-t}|\mathcal{H}_{t}\right]-\beta^{*}_t\right)|\mathcal{H}_{t-1}\right)\right]
,
\end{align*}
where $\beta_{t}^{*}=\mathrm{E}[\omega_{0:t}r_t]$. 
\end{lemma}
We note that setting 
$\zeta_{t}=(\beta^{*}_{t},0)$ in eq.~\eqref{eq:var_multi} recovers the above.
This suggests the following theorem. 

\begin{theorem}
\label{thm:drs2}
The estimator $\hat{\beta}^{T-1}_{\operatorname{reg}}$ is locally and intrinsically efficient. 
\end{theorem}

\begin{remark}
Practically, when the horizon is long, there may be too many parameters to optimize, which can causes overfitting. That is, although there is no inflation in MSE asymptotically, there may be issues in finite samples. To avoid this problem, some constraint or regularization should be imposed on the parameters. Here we will consider the estimator $\hat{\beta}^{k}_{\operatorname{reg}}\,(0\leq k \leq T-1)$ given by $\hat{\beta}_{d}(\{m_t(x_t,a_t;\hat{\zeta})\}_{t=0}^{T-1})$
for the constrained control variates: 
\begin{align*}
    m_t(x_t,a_t;\zeta) = \begin{cases}
    \zeta_{t1} + \zeta_{t2} q(x_t,a_t;\hat{\tau})\, (0 \leq t < k), \\
    \zeta_{k1} + \zeta_{k2} q(x_t,a_t;\hat{\tau})\,(k \leq t \leq T-1) .
    \end{cases}
\end{align*}
The estimator $\hat{\beta}^{T-1}_{\operatorname{reg}}$ corresponds to the originally introduced estimator. We can also obtain theoretical guarantees of $\hat{\beta}^{k}_{\operatorname{reg}}$ for  $k \neq T-1$. For details, see Appendix \ref{sec:pra2}. 
\end{remark}

\subsection{EMP for RL}
\label{sec:drss}

First, we define a control variate:
\begin{align*}
   g(\mathcal D_{x,a};\xi,\hat{\tau})=\sum_{t=0}^{T-1}\gamma^{t}\left(\omega_{0:t}m_t(x_t,a_t;\xi,\hat{\tau})-\omega_{0:t-1}\left \{\sum_{a\in A} m_t(x_t,a;\xi,\hat{\tau}){\pi_{e}}(a|x_t)\right \}\right).
\end{align*}
By setting $m_t(x_t,a_t;\xi,\hat{\tau})=\xi_{1t}+\xi_{2t} q(x_t,a_t;\hat{\tau})$, define $\hat{\xi}$;
\begin{align*}
    \hat{\xi}(\hat{\tau})=\argmax_{\xi \in \mathbb{R}^{2}}\mathrm{E}_{n}[\log \{1+g(\mathcal D_{x,a};\xi,\hat{\tau})\}].
\end{align*}
Then, an estimator $\hat{\beta}^{T-1}_{\mathrm{emp}}$ is defined as 
\begin{align*}
\hat{\beta}^{T-1}_{\mathrm{emp}}=\mathrm{E}_{n}\left[\sum_{t=0}^{T-1}\omega_{0:t}\gamma^{t}r_{t}\frac{\hat{c}^{-1}}{1+g(\mathcal D_{x,a};\hat{\xi},\hat{\tau})}\right],\,\hat{c}=\mathrm{E}_{n}\left[\frac{1}{1+g(\mathcal D_{x,a};\hat{\xi},\hat{\tau})}\right]. 
\end{align*}
This estimator has the same efficiencies as $\hat{\beta}^{T-1}_{\mathrm{reg}}$ because the asymptotic MSE is the same. Importantly, the estimator $\hat{\beta}^{T-1}_{\mathrm{emp}}$ also satisfies a 1-boundedness and stability.

\begin{theorem}
\label{thm:drss2}
The asymptotic MSE of the estimator $\hat{\beta}^{T-1}_{\mathrm{emp}}$ is the same as that of  $\hat{\beta}^{T-1}_{\mathrm{reg}}$. Hence, it is also locally and intrinsically efficient. It also satisfies 1-boundeness and stability.  
\end{theorem}

\begin{table}[!]
\setlength{\tabcolsep}{0.75em}%
    \centering
    \caption{SatImage ($\operatorname{RMSE}\times 1000$ )}
    \begin{tabular}{ccccccccc}
    \toprule
    Behavior policy & DM1 & DM2&  IS & SNIS & DR  &MDR & REG & EMP   \\
    \midrule
    $0.7\pi_d+0.3\pi_u$ &  18.1 & 12.2  & 6.7 & 4.0  & 3.0  & 3.8 & \textbf{2.8} & \textbf{2.8}  \\
    $0.4\pi_d+0.6\pi_u$  & 49.2  & 30.5  & 12.0 & 5.6  & 5.0  & 5.3 & \textbf{4.4}  & \textbf{4.4}  \\
    $0.0\pi_d+1.0\pi_u$ & 128.6  & 71.7  & 26.0  &  \textbf{12.7} & 18.0  & 14.4  & \textbf{13.6} & 13.7 \\\bottomrule
    \end{tabular}\vspace{0.7em}
    \label{tab:satimage}
    \centering
    \caption{Pageblock ($\operatorname{RMSE}\times 1000$ )}
    \begin{tabular}{ccccccccccc}\toprule
    Behavior policy & DM1 & DM2& IS & SNIS & DR  & MDR & REG   & EMP  \\ \midrule
     $0.7\pi_d+0.3\pi_u$  & 21.8 & 2.6  & 8.5 & 3.4    & \textbf{1.4} & 2.3 & 1.5   & \textbf{1.4} \\
    $0.4\pi_d+0.6\pi_u$   & 32.4 & 5.6  & 13.4 & 4.0  & 2.7   & 3.4 & \textbf{2.5}  & \textbf{2.4}   \\ 
    $0.0\pi_d+1.0\pi_u$ & 62.0
     & 16.0 & 27.2 & 6.5  & 7.2  &  6.4 & \textbf{4.9}  & \textbf{4.9} 
     \\\bottomrule
    \end{tabular}\vspace{0.7em}
    \label{tab:Pageblock}
    \centering
    \caption{PenDigits ($\operatorname{RMSE}\times 1000$ )}
     \begin{tabular}{ccccccccccc}\toprule
    Behavior policy & DM1 &DM2  & IS & SNIS & DR  &MDR & REG  & EMP  \\ \midrule
    $0.7\pi_d+0.3\pi_u$  & 8.1  &  8.2 &  6.1 & 2.8  & 1.5 & 2.2  & \textbf{1.4}  & \textbf{1.4}  \\
    $0.4\pi_d+0.6\pi_u$  & 19.4  &  17.4 &  10.7 & 3.9 & 2.2  & 3.4  & \textbf{2.1}  & \textbf{2.0}  \\
   $0.0\pi_d+1.0\pi_u$ & 58.6  & 56.0  & 29.6  & 9.9 & 11.1   & \textbf{9.4} & \textbf{9.4}  & 9.5
   \\\bottomrule
    \end{tabular}
    \label{tab:Pendigits}
    \vspace{-0.9em}
\end{table}
\section{Experiments}
\label{sec:experiment}

\subsection{Contextual Bandit}\label{sec:expcb}

We evaluate the OPE algorithms using the standard classification data-sets from the UCI repository. Here, we follow the same procedure of transforming a classification data-set into a contextual bandit data set as in \citep{DudikMiroslav2014DRPE,Chow2018}. Additional details of the experimental setup are given in Appendix~\ref{sec:experiment-ape}. 

We first split the data into training and evaluation.
We make a deterministic policy $\pi_d$ by training a logistic regression classifier on the training data set. Then, we construct evaluation and behavior policies as mixtures of $\pi_d$ and the uniform random policy $\pi_u$. The evaluation policy $\pi_e$ is fixed at $0.9\pi_d+0.1\pi_u$. Three different behavior policies are investigated by changing a mixture parameter. 

Here, we compare the (practical) REG and EMP with DM, SIS, SNIS, DR, and MDR on the evaluation data set. First, two Q-functions $\hat{q}_1(x,a),\hat{q}_2(x,a)$ are constructed by fitting a logistic regression in two ways with a $l1$ or $l2$ regularization term. We refer them as DM1 and DM2. Then, in DR, we use a mixture of Q-functions $0.5\hat{q}_1+0.5\hat{q}_2 $ as $m(x,a)$. For MDR, we use a logistic function as $m(x,a)$ and we use SGD to solve the resulting non-convex high-dimensional optimization (e.g., for SatImage we have $6(\mathrm{number\,of\,actions})\times 36 (\mathrm{number\,of\,covariates})$ parameters). 
We use $m(x,a;\zeta)=\zeta^{\top}(1,\hat{q}_1,\hat{q}_2)$ in REG and $m(x,a;\xi)=\xi^{\top}(1,\hat{q}_1,\hat{q}_2)$ in EMP.

The resulting estimation RMSEs (root mean square error) over 200 replications of each experiment are given in Tables \ref{tab:satimage}--\ref{tab:Pendigits}, where we highlight in bold the best two methods in each case. We first find that REG and EMP generally have overall the best performance. Second we see that this arises because they achieve similar RMSE to SNIS when SNIS performs well and similar RMSE to (M)DR when (M)DR performs well, which is thanks to the intrinsic efficiency property.
Whereas REG's and EMP's intrinsic efficiency is visible, MDR still often does slightly worse than DR despites its partial intrinsic efficiency,
which can be attributed to optimizing too many parameters leading to overfitting in the sample size studied.

\subsection{Reinforcement Learning}\label{sec:exprl}

We next compare the OPE algorithms in three standard RL setting from OpenAI Gym \citep{gym}: Windy GridWorld, Cliff Walking, and Mountain Car. 
For further detail on each see Appendix \ref{sec:experiment-ape}.
We again split the data into training and evaluation.
In each setting we consider varying evaluation dataset sizes.
In each setting, a policy $\pi_{d}$ is computed as the optimal policy of the MDP based on the training data using Q-learning. The evaluation policy $\pi_e$ is then set to be $(1-\alpha)\pi_{d}+\alpha \pi_u$, where $\alpha=0.1$. The behavior policy is defined similarly with $\alpha=0.2$ for Windy GridWorld and Cliff Walking and with $\alpha=0.15$ for Mountain Car. We set the discounting factor to be $1.0$ as in \citep{Chow2018}. 

We compare the (practical) REG, EMP with $k=2$ with DM, SIS, SNSIS, DR, MDR on the evaluation data set generated by a behavior policy. A Q-function model is constructed using an off-policy TD learning \citep{SuttonRichardS1998Rl:a}. This is used in DM, DR, REG, and EMP. For MDR, we use a linear function for $m(x,a)$ in order to enable tractable optimization given the many parameters due to long horizons. 

We report the resulting estimation RMSEs over 200 replications of each experiment in Tables \ref{tab:windy}--\ref{tab:mou}. 
We find that the modest benefits we gained in one time step in the CB setting translate to significant outright benefits in the longer horizon RL setting. REG and EMP consistently outperform other methods. Their RMSEs are indistinguishable except for one setting where EMP has slightly better RMSE. These results highlight how the theoretical properties of intrinsic efficiency, stability, and boundedness can translate to improved performance in practice.

\begin{table}[!]
\setlength{\tabcolsep}{0.75em}%
    \centering
    {
    \caption{Windy GridWorld (RMSE)}
    \begin{tabular}{cccccccc}\toprule
     Size & DM  & SIS & SNSIS & DR & MDR  &  REG & EMP   \\ \midrule
     250  & 2.9   & 0.64 & 0.49 & 0.17 & 0.28 & \textbf{0.09} & \textbf{0.09}    \\
     500  &  2.8 & 0.53 & 0.34 & 0.11 & 0.21  & \textbf{0.06} & \textbf{0.06}  \\
    750  & 2.6 & 0.39 & 0.29 & 0.09 & 0.14  & \textbf{0.05} & \textbf{0.05}
    \\\bottomrule
    \end{tabular}
    \label{tab:windy}
    }\vspace{0.4em}
    \centering
    {
    \caption{Cliff Walking (RMSE)}
    \begin{tabular}{cccccccc}\toprule
     Size & DM  & SIS & SNSIS & DR & MDR & REG & EMP   \\ \midrule
    1000 &  7.7 & 3.6  & 2.9  & 2.5  & 2.3 &  \textbf{2.1} & \textbf{2.1} \\ 
    2000 &  6.0 & 3.2  & 2.4  & 2.3  & 2.2 & \textbf{1.6} & \textbf{1.5} \\ 
    3000 &  6.8 & 3.1  & 2.2  & 2.2  & 2.0  & \textbf{1.2} & \textbf{1.1}
    \\\bottomrule
    \end{tabular}
    \label{tab:cliff}
    }\vspace{0.4em}
    \centering
    \caption{Mountain Car (RMSE)}
    \begin{tabular}{cccccccc}\toprule
     Size & DM  & SIS & SNSIS & DR & MDR & REG & EMP   \\ \midrule
    1000 & 9.8  &  4.2   & 3.7  & 1.9   &  1.9  &  \textbf{1.7} & \textbf{1.7}  \\ 
    2000 & 10.6 & 3.3  & 2.9  & 1.6  & 1.6  & \textbf{1.2}  & \textbf{1.2}  \\ 
    3000 & 8.2  & 2.4   & 1.8    &  1.4 &  1.5  & \textbf{1.0}  & \textbf{1.0}
    \\\bottomrule
    \end{tabular}
    \label{tab:mou}
    \vspace{-0.9em}
\end{table}

\section{Conclusion and Discussion}

We studied various desirable properties for OPE in CB and RL. Finding that no existing estimator satisfies all of them, we proposed two new estimators, REG and EMP, that satisfy consistency, local efficiency, intrinsic efficiency, 1-boundedness, and stability. These theoretical properties also translated to improved comparative performance in a variety of CB and RL experiments.

In practice, there may be additional modifications that can further improve these estimators.
For example, \cite{wang2017optimal,thomas2016} propose hybrid estimators that blend or switch to DM when importance weights are very large. This reportedly works very well in practice but may make the estimator inconsistent under misspecification unless blending vanishes with $n$. 
In this paper, we focused on consistent estimators.
Also these do not satisfy intrinsic efficiency, 1-boudedness, or stability. 
Achieving these properties with blending estimators remains an important next step.

\bibliography{pfi}
\bibliographystyle{abbrvnat}

\newpage 

\appendix

\begin{table}[]
    \centering
     \caption{Summary of notations}
    \begin{tabular}{l|l}
    ${\pi_{e}}(a|x)$     & Target policy \\
    ${\pi_{b}}(a|x)$     & Exploration policy \\
    $\beta^{*}$ & Parameter of interest $\beta^{\pi_e}_{T}$\\
    $\mathbb{P}$,\,$\mathrm{E}[\cdot]$ &  Expectation with respect to a behavior policy \\
    $\mathrm{var}[\cdot] $ & Variance \\
    $\mathrm{Asmse}[\cdot] $ & Asymptotic variance \\
    $\mathbb{P}_{n}$,\,$\mathrm{E}_{n}$& Empirical approximation based on a set of samples from a behavior policy \\
    $\mathbb{G}_{n}$ & Empirical process $\sqrt{n}(\mathbb{P}_{n}-\mathbb{P})$ \\
    $q(x,a;\tau)$ & Model for Q-function with parameter $\tau$ \\
    $\omega_{t_1:t_2}$ & Cumulative importance ratio $\prod_{t=t_1}^{t_2} {\pi_{e}}(a_t|x_t)/{\pi_{b}}(a_t|x_t)$  \\
    $\zeta$ & Parameter in $m(x)$ for REG, SNREG \\
    $\xi$ & Parameter in $m(x)$ for EMP \\
    $R_{\mathrm{max}}$ & An upper bound of the reward function \\
    $\mathcal{H}_{T-1}$ & $(x_0,a_0,r_0,\cdots,x_{T-1},a_{T-1},r_{T-1})$ in T-step trajectory \\
    $x^{(i)}$ & $i$-th sample  \\
    $\stackrel{p}{\rightarrow}$ & Convergence in probability
    \end{tabular}
\end{table}
\newpage

\section{SNREG (self-normalized REG)}
\label{sec:bdrs}

Herein, we construct an estimator exhibiting partial intrinsic efficiency, 2-boundedness and partial stability based on a self-normalized estimator \citep{RobinsJames2007CPoD,thomas2016}. The partial intrinsic efficiency means that the resulting estimator's asymptotic MSE is smaller than SNDR and SNIS. Further, partial stability is defined as follows.

\begin{definition}[Partial stability] 
\label{def:bound2} An estimator satisfies the stability when $\hat{\tau}$ does not depend on the reward. 
\end{definition}

This condition indicates that the variance can be still bounded after defining the ratio and the estimated Q-function. The DM, SNDR have been easily proved to have this property. In addition, in the following proof section, we prove that the practical EMP also possesses this property.

Consider a family of unbiased estimators: $\hat{\beta}_{\operatorname{snd}}(m)$ as a solution to 
\begin{align*}
   \mathrm{E}_{n}\left[\beta-\left \{\sum_{a \in A}m(x,a)\pi_{e}(a|x)\right \} -\frac{\omega(a,x)}{\mathrm{E}_{n}[\omega(a,x)] }\{r-m(x,a)\}\right]=0,
\end{align*}
where $\pi_{e}(a|x)/\pi_{b}(a|x)=\omega(a,x)$. 
The SNDR estimator is subsequently defined as $\hat{\beta}_{\mathrm{sndr}}=\hat{\beta}_{\mathrm{snd}}(q(x,a;\hat{\tau}))$.
First, the range of this estimator is $[0,2R_{\mathrm{max}}]$. Therefore, tihs satisfies 2-boundedness and partial stability. In addition, this satisfies the consistency for an arbitrary choice of $m(x,a)$. By selecting $\zeta_1+\zeta_2 q(a,x;\tau)$ as $m(x,a)$, this class is also observed to include a SNIS estimator setting $\zeta=(1,0)$, and a SNDR estimator setting $\zeta=(0,1)$. However, this class does not include an IS estimator.

The asymptotic MSE is calculated as follows. 
\begin{theorem}
\label{thm:bdrs}
The term $\mathrm{Asmse}[\hat{\beta}_{\operatorname{snd}}]$ is $n^{-1}V_{\operatorname{snd}}(m)$, where $V_{\operatorname{snd}}(m)$ is 
\begin{align*}
    &\mathrm{var}\left[\omega(a,x)\left(r-m(x,a)\right)-\left \{\sum_{a \in A}m(x,a)\pi_{e}(a|x)\right\}\right] \\
    &+\mathrm{E}\left[\omega(a,x)(r-m(x,a))\right]^{2}\mathrm{var}\left[\omega(a,x)\right]\\
    &-2\left(\mathrm{E}\left[w(a,x)^{2}(r-m(x,a))-\omega(a,x)\sum_{a\in A} {\pi_{e}} (a|x)m(a,x)\right]-\beta^{*}\right)\mathrm{E}\left[\omega(a,x)(r-m(x,a))\right]. 
\end{align*}
\end{theorem}

By minimizing the empirical approximation of the aforementioned asymptotic MSE with respect to $\zeta_1,\zeta_2$ and $\tau$ and plugging-in as
\begin{align*}
    (\hat{\zeta},\hat{\tau})=\argmin_{\zeta \in \mathbb{R}^{2},\tau \in \Theta_\tau} \hat{V}_{\operatorname{snd}}(m(x,a;\zeta,\tau)),
\end{align*}
we obtain the estimator $\hat{\beta}_{\operatorname{snreg}}=\hat{\beta}_{\operatorname{snd}}(m(x,a;\hat{\zeta},\hat{\tau}))$. 
Here, $(\hat{\zeta},\hat{\tau})$ converges in probability to $(\zeta^{*},{\tau}^{*})$
\begin{align}
\label{eq:snis_opt}
    (\zeta^{*},\tau^{*})=\argmin_{\zeta \in \mathbb{R}^{2},\tau \in \Theta_\tau} V_{\operatorname{snd}}(m(x,a;\zeta,\tau)).
\end{align}
The asymptotic MSE of $\hat{\beta}_{\operatorname{snreg}}$ is given as follows.

\begin{theorem}
\label{thm:bdrs2}
Under the assumption that the optimization problem in \eqref{eq:snis_opt} has a unique solution, 
\begin{align*}
    \mathrm{Asmse}[\hat{\beta}_{\operatorname{snreg}}]=n^{-1}\min_{\zeta \in \mathbb{R}^{2},\tau \in \Theta_\tau} V_{\operatorname{snd}}(m(x,a;\zeta,\tau)). 
\end{align*}
The asymptotic MSE is smaller than those of the SNIS and SNDR. 
\end{theorem}

\begin{theorem}
The estimator $\hat{\beta}_{\operatorname{snreg}}$ is locally efficient. 
\end{theorem}
\begin{proof}
The variance reaches an efficiency bound:  $\zeta_1=0$,\,$\zeta_2=1$ and $\tau=\tau^{*}$, noting
\begin{align*}
    \mathrm{E}[w(a,x)\{r-m(x,a)\}]=0. 
\end{align*}
\end{proof}

Table \ref{tab:satimage2}-\ref{tab:Pendigits2} shows the experimental result of SNREG. The performance of SNREG is quite similar to those of REG, SNREG and EMP. 

\begin{table}[!]
    \centering
    \caption{SatImage ($\times 1000$ )}
    \begin{tabular}{cccccccc}
    Behavior policy &  DR & SNDR &MDR & REG & SNREG & EMP   \\ \hline
    $0.7\pi_d+0.3\pi_u$   & 3.0 & 3.0 & 3.8 & 2.8 & 2.8 & 2.8 &   \\
    $0.4\pi_d+0.6\pi_u$   & 5.0 & 5.0 & 5.3 & 4.4 & 4.4 & 4.4  \\
    $0.0\pi_d+1.0\pi_u$ & 18.0 & 17.8  & 14.4 & 13.6 & 13.6 & 13.7 
    \end{tabular}
    \label{tab:satimage2}
\end{table}

\begin{table}[!]
    \centering
    \caption{Pageblock ($\times 1000$ )}
    \begin{tabular}{ccccccc}
     Behavior policy &  DR & SNDR & MDR & REG & SNREG & EMP  \\ \hline
     $0.7\pi_d+0.3\pi_u$   & 1.4 & 1.4 & 2.3 & 1.5  & 1.4 & 1.4 \\
    $0.4\pi_d+0.6\pi_u$   & 2.7 & 2.6  & 3.4 & 2.5  & 2.5 & 2.4   \\ 
    $0.0\pi_d+1.0\pi_u$   & 7.2  & 7.3 &  6.4 & 4.9 &  4.9  & 4.9 
    \end{tabular}
    \label{tab:Pageblock2}
\end{table}

\begin{table}[!]
    \centering
    \caption{PenDigits($\times 1000$ )}
     \begin{tabular}{ccccccc}
    Behavior policy & DR & SNDR &MDR & REG & SNREG & EMP  \\ \hline
    $0.7\pi_d+0.3\pi_u$   & 1.5 & 1.5 & 2.2  & 1.4 & 1.4 & 1.4  \\
    $0.4\pi_d+0.6\pi_u$  & 2.2 & 2.2  & 3.4  & 2.1 & 2.1 & 2.0  \\
   $0.0\pi_d+1.0\pi_u$ & 11.1  & 10.8 & 9.4 & 9.4 & 9.4 & 9.5
    \end{tabular}
    \label{tab:Pendigits2}
\end{table}

\section{Theoretical property of $\hat{\beta}^{0}_{\operatorname{reg}}$}
\label{sec:pra2}

Herein, we provide some theoretical property of $\hat{\beta}^{0}_{\operatorname{reg}}$. In fact, the variance of $\hat{\beta}^{0}_{\operatorname{reg}}$ is smaller than the following estimator: 
\begin{align*}
    \hat{\beta}_{\operatorname{sn2sis}}=\mathrm{E}_{n}\left[\sum_{t=0}^{T-1}\frac{\omega_{0:t}}{\mathrm{E}_{n}[\omega_{0:T-1}]}\gamma^{t}r_{t}\right].
\end{align*}
The difference between this estimator and $\hat{\beta}_{\operatorname{snsis}}$ is that the denominator is $\mathrm{E}_{n}[\omega_{0:T-1}]$ instead of $\mathrm{E}_{n}[\omega_{0:t-1}]$. 
\begin{theorem}
\label{thm:drs1}
The asymptotic MSE of $\hat{\beta}^{0}_{\operatorname{reg}}$ is smaller than those of $\hat{\beta}_{\operatorname{sis}}$,\,$\hat{\beta}_{\operatorname{sn2is}}$ and  $\hat{\beta}_{\operatorname{dr}}$.
\end{theorem}

\section{Proofs}
The assumption is as follows. 
\begin{assumption}
(a1) Parameter space $\Theta_{\tau}$ is compact and sufficiently large, (a2) the term $|q(x,a;\tau)|\leq R_{\mathrm{max}}$, (a3) the optimal solution $(\zeta^{*},\tau^{*})$ in \eqref{eq:optimzation} is unique. 
\end{assumption}
Note that we have assumed (a1) and (a2) for all of theorems. Regarding (a3), we have assumed for Theorem \ref{thm:drs}. In addition, we have assumed that the reward $r$ and the cumulative ratio $w_{t_1:t_2}$ are bounded in the main paper. These condition (uniform boundedness of reward and cumulative ratio) can be relaxed to each theorem when discussing asymptotic properties. However, for simplicity, we assumed these conditions.

\label{sec:proof}

\begin{proof}[Proof of Theorem \ref{thm:drs}]

We denote $m^{*}=\zeta^{*}_1+\zeta^{*}_2 q(x,a;\tau^{*})$, $\hat{m}=\hat{\zeta}_1+\hat{\zeta}_2 q(x,a;\hat{\tau})$ and $u(m)$ as
\begin{align*}
   \omega(a,x)r+\left\{\sum_{a \in A}m(x,a)\pi_{e}(a|x)\right \} -\omega(a,x) m(x,a).
\end{align*} 

We prove two lemmas first.
\begin{lemma}
\label{lem:con}
$\hat{\zeta}\stackrel{p}{\rightarrow} \zeta^{*}$ and  $\hat{\tau}\stackrel{p}{\rightarrow} \tau^{*}$.
\end{lemma}

\begin{proof}
First, we define a space $\Theta_\zeta$, which always includes $\hat{\zeta}$. We can take a compact set as $\Theta_\zeta$ noting that  is uniquely defined fixing $\tau$, $\hat{\zeta}$ and the all of assumptions. 

Then, based on Lemma 2.4 in \cite{newey94}, an uniform convergence condition:
\begin{align*}
    \sup_{\tau \in \Theta_\tau,\zeta \in \Theta_\tau}|(\mathbb{P}_{n}-\mathbb{P})\left\{u(\zeta_1+\zeta_2 q(x,a;\tau))\right\}^{2}|\stackrel{p}{\rightarrow}0
\end{align*}
is satisfied using an assumption (a1) and the fact from (a2) that $u(\zeta_1+\zeta_2 q(x,a;\tau))^{2}$ is bounded uniformly over $\zeta \in \mathbb{R}^{2}$ and $\tau \in \Theta_{\tau}$. 

Then, by using Theorem 5.7 in \cite{VaartA.W.vander1998As}, the statement holds from (a1), (a3) and the above uniform convergence condition. 
\end{proof}

\begin{lemma}
\label{lem:con2}
$\mathbb{G}_{n}[u(\hat{m})]-\mathbb{G}_{n}[u(m^{*})]=o_{p}(1)$.
\end{lemma}

\begin{proof}
Based on Lemma 19.24 in \cite{VaartA.W.vander1998As}, we have to confirm two statements; (1): for some $\delta>0$, the class $\{u(\zeta_1+\zeta_2 q(x,a;\tau)); |\zeta-\zeta^{*}|<\delta, |\tau-\tau^{*}|<\delta \}$ is a Donsker class, (2) the term
$\mathrm{E}[\left(u(\hat{m})-u(m^{*}) \right)^{2}]$ converges in probability to $0$. 

The first condition is satisfied using the assumption (a1) and the fact from (a2) that $u(\zeta_1+\zeta_2 q(x,a;\tau))$ is bounded uniformly over $\zeta \in \mathbb{R}^{2}$ and $\tau \in \Theta_{\tau}$, based on Example 19.7 in \cite{VaartA.W.vander1998As}. 

The second condition is satisfied as follows. First, $\hat{m}$ converges in probability to $m^{*}$ from Lemma \ref{lem:con} by continuous mapping theorem. In addition, $\{u(\zeta_1+\zeta_2 q(x,a;\tau));\zeta \in \mathbb{R}^{2}, \tau \in \Theta_\tau \}$ is uniformly integrable from the assumption (a2). Then, it is verified by Lebesgue convergence theorem.
\end{proof}

We go back to the main proof. Here, we want to know the behavior of $\sqrt{n}(u(\hat{m})-\beta^{*})$. 
This is decomposed as 
\begin{align*}
    \sqrt{n}(u(\hat{m})-\beta^{*})&=\mathbb{G}_{n}[u(\hat{m})]-\mathbb{G}_{n}[u(m^{*})]  \\
    &+\mathbb{G}_{n}[u(m^{*})] \\
    &+\sqrt{n}(\mathrm{E}[u(\hat{m})] -\beta^{*}). 
\end{align*}
The first term is $\mathrm{o}_{p}(1)$ by Lemma \ref{lem:con2}. The third term $\sqrt{n}(\mathrm{E}[u(\hat{m})] -\beta^{*})$ is  $0$ from the construction.
Then, it is found that the influence function of the estimator is $u(m^{*})$, that is,
\begin{align*}
       \sqrt{n}(u(\hat{m})-\beta^{*})&=\mathbb{G}_{n}[u(m^{*})]+\mathrm{o}_{p}(1). 
\end{align*}
Thus, the asymptotic MSE of $\hat{\beta}_{d}(\hat{\zeta}_1+\hat{\zeta}_2 q(x,a;\hat{\tau}))$ is the same as the variance of $\hat{\beta}_{d}(\zeta^{*}_1+\zeta^{*}_2 q(x,a;\tau^{*}))$. 
This concludes the proof. 
\end{proof}

\begin{proof}[Proof of Corollary \ref{col:drs_imp}]
We prove each statement as follows. \\

\textbf{Local efficiency}
By setting $\zeta=(0,1)$,\,$\tau=\tau^{*}$ in Theorem \ref{thm:drs}, it achieves the efficiency bound.  \\

It is obvious because the asymptotic variance of 
$\hat{\beta}_{\mathrm{reg}}$ estimator is represented as 
\begin{align*}
n^{-1}\argmin_{\zeta \in \mathbb{R}^{2},\tau \in \Theta_\tau} \mathrm{E}[\{wr-\mathcal{F}(\zeta_1+\zeta_2 q(x;\tau))\}^{2}].     
\end{align*}

\textbf{Intrinsic efficiency}
We notice that the asymptotic variance of each estimator is represented as $n^{-1}\mathrm{E}[\{wr-\mathcal{F}(\zeta_1+\zeta_2 q(x;\tau))\}^{2}]$.
The SIS estimator corresponds to the case $\zeta=(0,0)$. The SNSIS estimator corresponds to the case $\zeta=(
\beta^{*},0)$. The DR estimator corresponds to the case $\zeta=(0,1)$ and $
\tau=\tau^{\dagger}$, where $\tau^{\dagger}$ is some convergence point of $\hat{\tau}$. 
\end{proof}

\begin{proof}[Proof of Lemma \ref{lem:drss}]

Because of the first order condition in \eqref{eq:drss_optim}, the following equation holds:
\begin{align*}
   \sum_{i=1}^{n}\hat{\kappa}^{(i)}\pi_{b}(a^{(i)}|x^{(i)})(w(x^{(i)},a^{(i)})-1)=0,
\end{align*}
where $\hat{\kappa}^{(i)}=\hat{\kappa}(a^{(i)}|x^{(i)};\hat{\xi},\hat{\tau})$. Then,
\begin{align*}
   \sum_{i=1}^{n}\hat{\kappa}^{(i)}(\pi_{e}(a^{(i)}|x^{(i)})-\pi_{b}(a^{(i)}|x^{(i)}))=0
\end{align*}

Regarding the 1-boundedness, it is proved as follows. 
\begin{align*}
\hat{\beta}_{\mathrm{emp}} &=
    \frac{1}{n}\sum_{i=1}^{n}\hat{c}(\mathcal D_{x,a};\hat{\xi},\hat{\tau})^{-1}\hat{\kappa}(\mathcal D_{x,a};\hat{\xi},\hat{\tau})\pi_{e}(a^{(i)}|x^{(i)})r^{(i)}
    \\
    & \leq  \frac{1}{n}\sum_{i=1}^{n}\hat{c}(\mathcal D_{x,a};\hat{\xi},\hat{\tau})^{-1}\hat{\kappa}(\mathcal D_{x,a};\hat{\xi},\hat{\tau})\pi_{e}(a^{(i)}|x^{(i)})R_{\mathrm{max}}\\
    &= R_{\mathrm{max}}.
\end{align*}
From the third line to the fourth line, we use a definition of $\hat{c}$. 

Regarding the partial stability, noting $\hat{\xi}$ and $\hat{\tau}$ are a function of $\mathcal D_{x,a}$ base on the form of optimization problem \eqref{eq:drss_optim}, it is proved as follows;
\begin{align*}
    \mathrm{var}[\hat{\beta}_{\mathrm{emp}}|\mathcal D_{x,a}] &=
    \frac{1}{n}\sum_{i=1}^{n}\left \{\hat{c}(\mathcal D_{x,a};\hat{\xi},\hat{\tau})^{-1}\hat{\kappa}(\mathcal D_{x,a};\hat{\xi},\hat{\tau})\pi_{e}(a^{(i)}|x^{(i)})\right \}^{2}\mathrm{var}[ r^{(i)}|\mathcal D_{x,a}]
    \\
    & \leq  \frac{1}{n}\sum_{i=1}^{n}\left \{\hat{c}(\mathcal D_{x,a};\hat{\xi},\hat{\tau})^{-1}\hat{\kappa}(\mathcal D_{x,a};\hat{\xi},\hat{\tau})\pi_{e}(a^{(i)}|x^{(i)})\right \}^{2}\sigma^{2}\\
    &\leq \sigma^{2}.
\end{align*}
From the second line to the third line, we have used the fact that $\max_{b}\sum b_i^{2}$ such that $\sum b_i=1$ is 1. 
\end{proof}

\begin{proof}[Proof of Theorem \ref{thm:drss}]

First, we prove $\hat{\xi}\stackrel{p}{\rightarrow}0$ and $\hat{\tau}\stackrel{p}{\rightarrow}0$. Define $(\xi,\tau^{\top})^{\top}=\psi$. 
\begin{lemma}
$\hat{\psi}\stackrel{p}{\rightarrow}0$
\end{lemma}
\begin{proof}
We use Theorem 5.7 in \cite{VaartA.W.vander1998As}. Here, note that 
\begin{align*}
    \mathcal{F}(\xi + \tau^{\top}t(x,a))=\psi^{\top}g(x,a),
\end{align*}
where $g(x,a)=(\mathcal{F}(1),\mathcal{F}(t(x,a)))^{\top}$ and the estimator $\hat{\psi}$ is an M-estimator defined by maximizing: 
\begin{align*}
    \mathrm{E}_{n}[\log(1+\psi^{\top}g(x,a))].
\end{align*}

The uniform convergence condition is proved similarly as the proof in Theorem \ref{thm:drs} based on (a1) and (a2). What we have to show is $\mathrm{E}[\log(1+\psi^{\top}g(x,a))]$ takes a maximum over $\psi \in \mathbb{R}^{d_\psi}$ if and only if $\psi=0$. This comes from the Jensen inequality: 
\begin{align*}
    \mathrm{E}[\log(1+\psi^{\top}g(x,a))]&\leq \log \mathrm{E}[(1+\psi^{\top}g(x,a))] \\
    &= \log\{1+\psi^{\top} \mathrm{E}[g(x,a)]\}=0 ,
\end{align*}
and a corresponding Hessian is a negative definite matrix. 
\end{proof}

Then, we can state that $\hat{c}$ also converges in probability to $1$

\begin{lemma}
$\hat{c}\stackrel{p}{\rightarrow}1$.
\end{lemma}
\begin{proof}
We have
\begin{align*}
  |\hat{c}-1|\leq |(\mathbb{P}_{n}-\mathbb{P})\{1+\mathcal{F}(m(x,a;\hat{\psi}))\}^{-1}  |+|\mathbb{P}[ \{1+\mathcal{F}(m(x,a;\hat{\psi}))\}^{-1}]-1|.
\end{align*}
The first term converges in probability to $0$ from the uniform convergence property based on the assumption (a1) and (a2). The second term also converges in probability to $0$ from the continuous mapping theorem, noting $\hat{\psi}\stackrel{p}{\rightarrow} 0$.
\end{proof}

Then, we show the following lemma.
\begin{lemma}
\label{lem:final}
\begin{align}
\sqrt{n}\left(\mathbb{P}_{n}\left(\frac{{\pi_{e}} r}{{\pi_{b}}'(\hat{c},\hat{\psi})}\right)-\beta^{*}\right) \nonumber 
&=\sqrt{n}\left(\mathbb{P}_{n}\left(\frac{{\pi_{e}} r}{{\pi_{b}}}-\psi^{*}g(x,a)\right)-\beta^{*} \right)+\op(1),\nonumber 
\end{align}
where 
${\pi_{b}}'(c,\psi)=c\pi_{b}(1+\psi^{\top}g(x,a))$ and  $\psi^{*}$ is defined as 
\begin{align*}
\psi^{*}  &=\argmin_{\psi \in \mathbb{R}^{d_{\phi}}}\mathrm{var}\left[\left\{\omega(a,x)r-\psi^{\top} g(x,a)\right\}\right]\\ &=\mathrm{E}[g(x,a)g(x,a)^{\top}]^{-1}\mathrm{E}\left[\frac{\pi_{e}}{\pi_b} r g(x,a)\right].
\end{align*}
\end{lemma}
\begin{proof}
We have
\begin{align}
&\sqrt{n}\left(\mathbb{P}_{n}\left(\frac{{\pi_{e}} r}{{\pi_{b}}'(\hat{c},\hat{\psi})}\right)-\beta^{*}\right) \nonumber \\
&=\left(\mathbb{G}_{n}\left(\frac{{\pi_{e}} r}{{\pi_{b}}'(\hat{c},\hat{\psi}) }\right) -\mathbb{G}_{n}\left(\frac{{\pi_{e}} r}{{\pi_{b}}}\right) \right)+\mathbb{G}_{n}\left(\frac{{\pi_{e}} r}{{\pi_{b}}}\right) +\sqrt{n}\left(\mathrm{E}\left[\frac{{\pi_{e}} r}{{\pi_{b}}'(\hat{c},\hat{\psi})}\right]-\beta^{*}\right) \label{eq:drss3}\\
&=\mathbb{G}_{n}\left(\frac{{\pi_{e}} r}{{\pi_{b}}}\right) +\sqrt{n}\left(\mathrm{E}\left[\frac{{\pi_{e}} r}{{\pi_{b}}'(\hat{c},\hat{\psi})}\right]-\beta^{*}\right)+\op(1) \label{eq:drss4}\\
&=\sqrt{n}\left(\mathbb{P}_{n}\left(\frac{{\pi_{e}} r}{{\pi_{b}}}-\psi^{*}g(x,a)\right)-\beta^{*} \right)+\op(1).\label{eq:drss5}
\end{align}

From the second line \eqref{eq:drss3} to the third line  \eqref{eq:drss4} , noting that ${\pi_{b}}'(\hat{c},\hat{\psi})$ converges in probability to $\pi_b$ from the fact $\hat{c}\stackrel{p}{\rightarrow}1$ and $\hat{\psi}\stackrel{p}{\rightarrow} (0,0)$ and (a1), (a2),
we used:
\begin{align*}
   \mathbb{G}_{n}\left(\frac{{\pi_{e}} r}{{\pi_{b}}'(\hat{c},\hat{\psi}) }\right)-\mathbb{G}_{n}\left(\frac{{\pi_{e}} r}{{\pi_{b}}}\right)=o_{p}(1). 
\end{align*}
From the third line  \eqref{eq:drss4} to the fourth line \eqref{eq:drss5} , we used the following argument. 
\begin{align}
\label{eq:mainmain}
 \sqrt{n}\mathrm{E}\left[\frac{{\pi_{e}} r}{{\pi_{b}}'(\hat{c},\hat{\psi})}\right]&=\sqrt{n}\left(\mathrm{E}\left[\nabla_{\psi^{\top}}\frac{{\pi_{e}} r}{{\pi_{b}}'}\right],\mathrm{E}\left[\nabla_{c}\frac{{\pi_{e}} r}{{\pi_{b}}'}\right]\right)|_{\psi^{*},c^{*}} ((\hat{\psi}-\psi^{*})^{\top},\hat{c}-c^{*})^{\top}+\mathrm{o}_{p}(1)  \\
 &=-\mathrm{E}\left[\frac{\pi_{e}}{\pi_b} r g^{\top}\right]\mathrm{E}[gg^{\top}]^{-1}\sqrt{n}\mathbb{P}_{n}g \label{eq:mainmain2} \\
 &=\sqrt{n}\mathbb{P}_{n}[-{\psi^{*}}^{\top}g] \nonumber. 
\end{align}
Here, from the first line \eqref{eq:mainmain} to the second line \eqref{eq:mainmain2}, we have used the fact that an estimator $\hat{\psi}$ and $\hat{c}$ are defined as an Z-estimator:
\begin{align*}
    \mathrm{E}_{n}\left[\frac{g}{1+\psi^{\top}g} \right]=0,\,
    \mathrm{E}_{n}\left[\frac{1}{1+\psi^{\top}g} -c\right]=0. 
\end{align*}
This implies 
\begin{align*}
    \sqrt{n}(\hat{\psi}-\psi^{*})&= -\mathrm{E}\left[\frac{g(x,a)g(x,a)^{\top}}{1+\psi^{\top}g(x,a) }\right]^{-1}\sqrt{n}\mathbb{P}_{n}g(x,a)|_{\psi^{*},c^{*}}+\op(1),\\
    &= -\mathrm{E}\left[g(x,a)g(x,a)^{\top} \right]^{-1}\sqrt{n}\mathbb{P}_{n}g(x,a)|_{\psi^{*},c^{*}}+\op(1), \\
    \sqrt{n}(\hat{c}-c^{*})&=-\mathrm{E}\left[\frac{g(x,a)}{1+\psi^{\top}g(x,a)}\right]\sqrt{n}(\hat{\psi}-\psi^{*})|_{\psi^{*},c^{*}}+\op(1)=\op(1). 
\end{align*}
\end{proof}

Finally, from Lemma \ref{lem:final}, the asymptotic variance of $\hat{\beta}_{\mathrm{emp}}$ is
\begin{align*}
n^{-1}\min_{\psi \in \mathbb{R}^{d_{\psi}} }\mathrm{var}\left[\left \{\omega(a,x)r-\psi g(x,a)\right\}\right].     
\end{align*}

\end{proof}

\begin{proof}[Proof of Theorem \ref{thm:pra_drs}]

We show an asymptotic statement for the practical $\hat{\beta}_{\mathrm{reg}}$ first. Then, we go to the asymptotic statement for the practical $\hat{\beta}_{\mathrm{emp}}$. 

We prove the following lemma first.
\begin{lemma}
\label{lem:con3}
$\hat{\zeta}\stackrel{p}{\rightarrow} \zeta^{*}$, where 
\begin{align}
   \zeta^{*}=  \argmin_{\zeta \in \mathbb{R}^{2} }\mathrm{E}\left[\left\{wr- \mathcal{F}(\zeta_1+\zeta_2 q(x,a;\tau^{\dagger}))\right\}^{2}\right]. 
\end{align}
\end{lemma}
\begin{proof}
We use Theorem 5.7 in \cite{VaartA.W.vander1998As}. The uniform convergence condition is proved similarly as the proof in Theorem \ref{thm:drs} based on (a1) and (a2). Therefore, what we have to prove is the minimum of the following function 
\begin{align}
   \zeta \to \mathrm{E}\left[\left\{wr- \mathcal{F}(\zeta_1+\zeta_2 q(x,a;\tau^{\dagger}))\right\}^{2}\right]
\end{align}
is uniquely defined. This is obvious because the above function is a quadratic function with respect to $\zeta$.
\end{proof}

For the rest of the proof, by following the same argument in the proof of Theorem \ref{thm:drs} with redefining 
\begin{align*}
    m^{*}=\zeta^{*}_1+\zeta^{*}_2 q(x,a;\tau^{\dagger}),
\end{align*}
the statement is proved.

Next, we show a statement for $\hat{\beta}_{\mathrm{emp}}$. As in the proof of Theorem \ref{thm:drss}, we show the following lemma.
\begin{lemma}
\begin{align}
\sqrt{n}\left(\mathbb{P}_{n}\left(\frac{{\pi_{e}} r}{{\pi_{b}}'(\hat{c},\hat{\xi},\hat{\tau})}\right)-\beta^{*}\right) \nonumber 
&=\sqrt{n}\left(\mathbb{P}_{n}\left(\frac{{\pi_{e}} r}{{\pi_{b}}}-\zeta^{*}(\hat{\tau})g(x,a;\hat{\tau})\right)-\beta^{*} \right)+\op(1),\nonumber 
\end{align}
where 
${\pi_{b}}'(c,\xi,\tau)=c\pi_{b}(1+\xi^{\top}g(x,a;\tau))$ and  $\zeta^{*}(\tau)$ is defined as 
\begin{align*}
\zeta^{*}(\tau)  &=\argmin_{\zeta \in \mathbb{R}^{2} }\mathrm{var}\left[\left(\omega(a,x)r-\zeta^{\top} g(x,a;\tau)\right)\right]\\ &=\mathrm{E}[g(x,a;\tau)g(x,a;\tau)^{\top}]^{-1}\mathrm{E}\left[\frac{\pi_{e}}{\pi_b} r g(x,a;\tau)\right].
\end{align*}
\end{lemma}

We go back to the main proof. Finally, we have  
\begin{align*}
&\sqrt{n}\left(\mathbb{P}_{n}\left(\frac{{\pi_{e}} r}{{\pi_{b}}}-\zeta^{*}(\hat{\tau})^{\top}g(x,a;\hat{\tau})\right)-\beta^{*} \right) \\
&=\mathbb{G}_{n}\left(\frac{{\pi_{e}} r}{{\pi_{b}}}-\zeta^{*}(\hat{\tau})^{\top}g(x,a;\hat{\tau})\right)-\mathbb{G}_{n}\left(\frac{{\pi_{e}} r}{{\pi_{b}}}-\zeta^{*}(\tau^{\dagger})^{\top}g(x,a;\tau^{\dagger})\right) \\
&+\mathbb{G}_{n}\left(\frac{{\pi_{e}} r}{{\pi_{b}}}-\zeta^{*}(\tau^{\dagger})^{\top}g(x,a;\tau^{\dagger})\right)+\sqrt{n}\left(\mathrm{E}\left[\frac{{\pi_{e}} r}{{\pi_{b}}}-\zeta^{*}(\hat{\tau})^{\top}g(x,a;\hat{\tau})\right]-\beta^{*}\right) \\
&=\mathbb{G}_{n}\left(\frac{{\pi_{e}} r}{{\pi_{b}}}-\zeta^{*}(\tau^{\dagger})^{\top}g(x,a;\tau^{\dagger})\right)+o_{p}(1).
\end{align*}
From the second line to the third line, we use an argument that the first term is equal to $\op(1)$ by the assumptions (a1), (a2) and the third term is $0$ from the construction. Therefore, the asymptotic variance (MSE) is
\begin{align*}
n^{-1}\min_{\zeta \in \mathbb{R}^{2}}\mathrm{var}\left[\left(\omega(a,x)r-\zeta^{\top} g(x,a;\tau^{\dagger})\right)\right].     
\end{align*}
\end{proof}

\begin{proof}[Proof of Theorem \ref{thm:rdrs}]

We use a law of total variance \citep{CliveG.Bowsher2012Isov}. 
\begin{align*}
&n\mathrm{var}\left[ \mathrm{E}_{n}\left[\sum_{t=0}^{T-1}\left(\gamma^{t}\omega_{0:t}r_{t}-\gamma^{t}\left(\omega_{0:t}m_t(x_t,a_t)-\omega_{0:t-1}\sum_{a\in A} m_t(x_t,a)\pi_{e}(a|x_t)\right) \right)\right]\right]\\
&=\sum_{t=0}^{T-1}\mathrm{E}\left[\mathrm{var}\left(\mathrm{E}\left[\sum_{k=0}^{T-1}\left(\gamma^{k}\omega_{0:k}r_{k}-\gamma^{k}(\omega_{0:k}m_k(x_k,a_k)-\omega_{0:k-1}\sum_{a\in A} m_k{\pi_{e}} )\right) |\mathcal{H}_{t}\right]|\mathcal{H}_{t-1}\right)\right]\\
&=\sum_{t=0}^{T-1}\mathrm{E}\left[\mathrm{var}\left(\mathrm{E}\left[\sum_{k=t}^{T-1}\left(\gamma^{k}\omega_{0:k}r_{k}-\gamma^{k}(\omega_{0:k}m_k(x_k,a_k)-\omega_{0:k-1}\sum_{a\in A} m_k{\pi_{e}} )\right) |\mathcal{H}_{t}\right]|\mathcal{H}_{t-1}\right)\right] \\
&=\sum_{t=0}^{T-1}\mathrm{E}\left[\gamma^{2t}\mathrm{var}\left(\mathrm{E}[\sum_{k=t}^{T-1}\gamma^{k-t}\omega_{0:k}r_{k}|\mathcal{H}_{t}]-(\omega_{0:t}m_t(x_t,a_t)-\omega_{0:t-1}\sum_{a\in A} m_t{\pi_{e}} ) |\mathcal{H}_{t-1}\right)\right]\\
&=\sum_{t=0}^{T-1}\mathrm{E}\left[\gamma^{2t}\omega_{0:t-1}^{2}\mathrm{var}\left(\mathrm{E}[\sum_{k=t}^{T-1}\gamma^{k-t}\omega_{t+1:k}r_{k}|\mathcal{H}_{t}]\omega_{t:t}-(\omega_{t:t}m_t(x_t,a_t)-\sum_{a\in A} m_t{\pi_{e}} ) |\mathcal{H}_{t-1}\right)\right].
\end{align*}
From the third line to the fourth line: 
\begin{align*}
    \mathrm{E}\left[\omega_{0:k}m_k(x_k,a_k)-\omega_{0:k-1}\sum_{a\in A} m_k(x_k,a)\pi_{e}(a|x_k)|\mathcal{H}_{t}\right]=0,
\end{align*}
for $k>t$. 
\end{proof}

\begin{proof}[Proof of Lemma \ref{lem:rbdrs}]

Define an estimator as a solution to: $\mathrm{E}_{n}[d,d_0,d_1,\cdots,d_{T-1}]^{\top}=0$, where 
\begin{align*}
d=\beta-\left \{\sum_{t=0}^{T-1}\omega_{0:t}\gamma^{t}r_{t}/c_t\right \},\,d_t=c_t-\omega_{0:t}.
\end{align*}
The asymptotic MSE of $(\hat{\beta},\hat{c}_1,\ldots,\hat{c}_{T-1}) $ is written as a sandwich formula: $n^{-1}A^{-1}B{A^{\top}}^{-1}$:
\begin{align*}
  A = \left(
    \begin{array}{cccc}
      1 & \gamma\beta^{*}_1 & \ldots & \gamma^{T-1}\beta^{*}_{T-1} \\
      0 & 1 & \ldots & 0 \\
      \vdots & \vdots & \ddots & \vdots \\
      0 & 0 & \ldots & 1
    \end{array}
  \right), \,
  B = \left(
    \begin{array}{cccc}
      \mathrm{var}[d] & \mathrm{cov}[d,d_1] & \ldots & \mathrm{cov}[d,d_{T-1}] \\
      \mathrm{cov}[d_1,d] & \mathrm{var}[d_1] & \ldots & 0 \\
      \vdots & \vdots & \ddots & \vdots \\
      \mathrm{cov}[d_{T-1},d] & 0 & \ldots & \mathrm{var}[{d_{T-1}}]
    \end{array}
  \right),
\end{align*}
where
\begin{align*}
    \beta^{*}_t = \mathrm{E}[\omega_{0:t}r_{t}].
\end{align*}
First, $A^{-1}$ is 
\begin{align*}
  A = \left(
    \begin{array}{cccc}
      1 & -\gamma \beta^{*}_1 & \ldots & -\gamma^{T-1}\beta^{*}_{T-1} \\
      0 & 1 & \ldots & 0 \\
      \vdots & \vdots & \ddots & \vdots \\
      0 & 0 & \ldots & 1
    \end{array}
    \right).
\end{align*}
Then, the (1,1) element in $A^{-1}B{A^{\top}}^{-1}$ is 
\begin{align}
\label{eq:1_1}
    \mathrm{var}[d]-\sum_{t=0}^{T-1}\gamma^{t}\beta^{*}_t\mathrm{cov}[d,d_t]+\sum_{t=0}^{T-1}\gamma^{2t}{\beta^{*}_t}^{2}\mathrm{var}[d_t].
\end{align}
First, $\mathrm{var}[d]$ is equal to 
\begin{align*}
   \sum_{t=0}^{T-1}\mathrm{E}\left[\gamma^{2t}\omega_{0:t-1}^{2}\mathrm{var}\left(\mathrm{E}\left[\sum_{k=t}^{T-1}\gamma^{k-t}\omega_{t+1:k}r_{k-t}|\mathcal{H}_{t}\right]\omega_{t:t} |\mathcal{H}_{t-1}\right)\right]. 
\end{align*}
Then, $\mathrm{cov}[d,d_t]$ is equal to 
\begin{align*}
    \mathrm{E}\left[\gamma^{2k}\omega_{0:t-1}^{2}\mathrm{cov}\left(\mathrm{E}\left[\sum_{k=t}^{T-1}\gamma^{k-t}\omega_{t+1:k}r_{k-t}|\mathcal{H}_{t}\right]\omega_{t:t},\beta^{*}_t\omega_{t:t}|\mathcal{H}_{t-1}\right)\right]. 
\end{align*}
Finally, the term \eqref{eq:1_1} is equal to 
\begin{align*}
    \sum_{t=0}^{T-1}\mathrm{E}\left[\gamma^{2t}\omega_{0:t-1}^{2}\mathrm{var}\left(\omega_{t:t}\left(\mathrm{E}[\sum_{k=t}^{T-1}\gamma^{k-t}\omega_{t+1:k}r_{k-t}|\mathcal{H}_{t}]-\beta^{*}_t\right)|\mathcal{H}_{t-1}\right)\right].
\end{align*}
\end{proof}

\begin{proof}[Proof of Theorem \ref{thm:drs2}]

As in the same way of Theorem \ref{thm:drs}, it is proved that the asymptotic MSE of $\hat{\beta}_{\mathrm{reg}}^{T-1}$ is 
\begin{align*}
    n^{-1}\min_{\zeta \in \mathbb{R}^{d_{\zeta}}}\mathrm{var}[v(\{\zeta_{1t}+\zeta_{2t}q(x,a;\tau^{\dagger})\}_{t=0}^{T-1})]. 
\end{align*}
We prove the intrinsic efficiency. Regarding local efficiency, they are proved as the proof of Corollary \ref{col:drs_imp}. The asymptotic MSEs of $\hat{\beta}_{\mathrm{sis}}$, $\hat{\beta}_{\mathrm{snsis}}$ and $\hat{\beta}_{\mathrm{dr}}$ are represented as a form of $n^{-1}\mathrm{var}[v(\{\zeta_{1t}+\zeta_{2t} q(x,a;\tau^{\dagger})\}_{t=0}^{T-1})]$. Setting $\zeta_{1t}=0$ and $\zeta_{2t}=0$, it corresponds to the estimator $\hat{\beta}_{\mathrm{sis}}$. Setting $\zeta_{1t}=\beta^{*}_t$ and $\zeta_{2t}=0$, it corresponds to the estimator $\hat{\beta}_{\mathrm{snsis}}$. Setting $\zeta_{1t}=0$ and $\zeta_{2t}=1$, it corresponds to the estimator $\hat{\beta}_{\mathrm{dr}}$. This concludes the intrinsic efficiency.

\end{proof}

\begin{proof}[Proof of Theorem \ref{thm:drss2}]

We prove a 1-boundedness and (partial) stability. When $\tau$ is not pre-estimated, it has stability. When $\tau$ is pre-estimated, it has partial stability. We prove the latter point. 
Regarding the asymptotic result, we can prove as in Theorem \ref{thm:drss}. 

From the first consider of optimization problem with respect to $\zeta_{1t}$ for $0\leq t\leq T-1$, we have 
\begin{align*}
    0 = \mathrm{E}_{n}\left[\frac{w_{0:t}-w_{0:t-1}}{1+g(\mathcal D_{x,a};\hat{\xi},\hat{\tau})}\right].
\end{align*}
Noting $w_{0:-1}=1$ for any $t$, 
\begin{align}
\label{eq:weight}
    0 = \mathrm{E}_{n}\left[\frac{w_{0:t}-1}{1+g(\mathcal D_{x,a};\hat{\xi},\hat{\tau})}\right].
\end{align}
The estimator $\hat{\beta}^{T-1}_{\mathrm{emp}}$ is bounded as follows. 
Regarding the 1-boundedness, 
\begin{align*}
\hat{\beta}^{T-1}_{\mathrm{emp}}& \leq  \frac{1}{n}\sum_{i=1}^{n}\sum_{t=0}^{T-1}\omega_{0:t}^{(i)}\gamma^{t}r^{(i)}_{t}\frac{\hat{c}^{-1}}{1+g(\mathcal D_{x,a};\hat{\xi},\hat{\tau})} \\
& \leq \frac{1}{n}\sum_{i=1}^{n}\sum_{t=0}^{T-1}\omega_{0:t}^{(i)}\gamma^{t}R_{\mathrm{max}}\frac{\hat{c}^{-1}}{1+g(\mathcal D_{x,a};\hat{\xi},\hat{\tau})}\\
&=\sum_{t=0}^{T-1}\gamma^{t}R_{\mathrm{max}}
\end{align*}
From the second to the third line, we have used \eqref{eq:weight}. 

Regarding the partial stability, noting that from the assumption, $\hat{\zeta}$ and $\hat{\tau}$ are functions of $\bx$ and $\ba$, 
\begin{align*}
\mathrm{var}[\hat{\beta}^{T-1}_{\mathrm{emp}}|\mathcal D_{x,a}]& \leq \mathrm{var} \left[\frac{1}{n}\sum_{i=1}^{n}\sum_{t=0}^{T-1}\omega_{0:t}^{(i)}\gamma^{t}r^{(i)}_{t}\frac{\hat{c}^{-1}}{1+g(\mathcal D_{x,a};\hat{\xi},\hat{\tau})}|\mathcal D_{x,a} \right] \\
& \leq \frac{1}{n}\sum_{i=1}^{n}\sum_{t=0}^{T-1}\left \{\omega_{0:t}^{(i)}\frac{\hat{c}^{-1}}{1+g(\mathcal D_{x,a};\hat{\xi},\hat{\tau})}\right\}^{2}\gamma^{2t}\mathrm{var}[r_t]\\
& \leq \frac{1}{n}\sum_{i=1}^{n}\sum_{t=0}^{T-1}\left \{\omega_{0:t}^{(i)}\frac{\hat{c}^{-1}}{1+g(\mathcal D_{x,a};\hat{\xi},\hat{\tau})}\right\}^{2}\gamma^{2t}\max[\mathrm{var}[r_t]] \\
&\leq \sum_{t=0}^{T-1}\gamma^{2t}\max[\mathrm{var}[r_t]]=\sigma^{2}. 
\end{align*}
From the third to the fourth line, we have used \eqref{eq:weight}. 

\end{proof}

\begin{proof}[Proof of Theorem \ref{thm:bdrs}]

The estimator is defined as a solution to the following equation with respect to $\beta$, $c$:
\begin{align*}
    \mathrm{E}_{n}[d_1,d_2]=0, 
\end{align*}
where 
\begin{align*}
d_1(x,a;\beta,c)=\beta-\frac{\omega_{0:0}(x,a)}{c}\left(r-m(x,a)\right)-\left \{\sum_{a \in A}m(x,a)\pi_{e}(a|x)\right\}\,,d_2(x,a;c)=c-\omega_{0:0}(x,a).
\end{align*}
The asymptotic MSE of $(\hat{\beta},\hat{c})$ is written as 
\begin{align*}
\mathrm{Asmse}[(\beta,c)^{\top}]&=
\begin{bmatrix}
1 & \mathrm{E}[\nabla_{c}d_1]  \\
0 & 1 \\
\end{bmatrix}^{-1}
\begin{bmatrix}
\mathrm{var}[d_1] & \mathrm{cov}[d_1,d_2]  \\
\mathrm{cov}[d_1,d_2] & \mathrm{var}[d_2]
\end{bmatrix}
\begin{bmatrix}
1 & 0\\
\mathrm{E}[\nabla_{c}d_1]  & 1 \\
\end{bmatrix}^{-1}|_{\beta^{*},c^{*}} \\
&= \begin{bmatrix}
1 & -\mathrm{E}[\nabla_{c}d_1] \\
0 & 1 \\
\end{bmatrix}
\begin{bmatrix}
\mathrm{var}[d_1] & \mathrm{cov}[d_1,d_2]  \\
\mathrm{cov}[d_1,d_2] & \mathrm{var}[d_2]
\end{bmatrix}
\begin{bmatrix}
1 & 0\\
-\mathrm{E}[\nabla_{c}d_1]\  & 1 \\
\end{bmatrix}|_{\beta^{*},c^{*}}.
\end{align*}

Therefore, the asymptotic MSE is given as 
\begin{align*}
    (\mathrm{var}[d_1]-2\mathrm{E}[\nabla_{c}d_1]\mathrm{cov}[d_1,d_2]+\mathrm{E}[\nabla_{c}d_1]^{2}\mathrm{var}[d_2])|_{\beta^{*},c^{*}}. 
\end{align*}
Here, noting that $c^{*}=1$, 
\begin{align*}
    \mathrm{E}[\nabla_{c}d_1]|_{c^{*}}&=\mathrm{E}\left[\omega(a,x)(r-m(x,a))\right],\\
    \mathrm{cov}[d_1,d_2]|_{c^{*}}&=\mathrm{E}\left[\omega_{0:0}^{2}(a,x)(r-m(x,a))-\omega(a,x)\left\{\sum_{a\in A} \pi_{e}(a|x)m(x,a)\right \}\right]-\beta^{*} ,\\
    \mathrm{var}[d_1]|_{\beta^{*},c^{*}}&=\mathrm{var}\left[\omega(a,x)\left(r-m(x,a)\right)-\left \{\sum_{a \in A}\pi_{e}(a|x)m(x,a)\right\}\right] ,\\
    \mathrm{var}[d_2]|_{\beta^{*},c^{*}}&=\mathrm{var}\left[\omega(a,x)\right]. 
\end{align*}
By combining all together, we get the conclusion. 

Note that the influence function is written as 
\begin{align}
\label{eq:inf}
    d_1(x,a)|_{\beta^{*},c^{*}}-\mathrm{E}[\nabla_{c}d_1(x,a)]|_{c^{*}}d_2(x,a)|_{\beta^{*},c^{*}}.
\end{align}
\end{proof}

\begin{proof}[Proof of Theorem \ref{thm:bdrs2}]

Define $u_{b}(c,m(x,a;\zeta,\tau))$: 
\begin{align*}
    \beta=\sum_{a \in A}m(x,a;\zeta,\tau){\pi_{e}}(a|x)+\frac{\omega_{0:0}(x,a)}{c}(r-m(x,a;
    \zeta,\tau)). 
\end{align*}
By noting that $\hat{c}\stackrel{p}{\to}c^{*}=1$, this is decomposed as 
\begin{align*}
    \sqrt{n}(u_{b}(\hat{c},\hat{m})-\beta^{*})&=\mathbb{G}_{n}[u_{b}(\hat{c},\hat{m})]-\mathbb{G}_{n}[u_{b}(1,m^{*})]  \\
    &+\mathbb{G}_{n}[u_{b}(1,m^{*})] \\
    &+\sqrt{n}(\mathrm{E}[u_{b}(\hat{c},\hat{m})] -\beta^{*}),
\end{align*}
when $m^{*}=\zeta^{*}_1+\zeta^{*}_2 q(x,a;\tau^{*})$ and $
\hat{m}=\hat{\zeta}_1+\hat{\zeta}_2 q(x,a;\hat{\tau})$. Here, the first term is equal to $\op(1)$ from assumptions (a1) and (a2). The last term is 
\begin{align*}
    &\sqrt{n}(\mathrm{E}[u_{b}(\hat{c},\hat{m})] -\beta^{*})\\
    &=\sqrt{n}(\mathrm{E}[u_{b}(1,\hat{m})]-\beta^{*})+\sqrt{n}\mathrm{E}[\nabla_{c}u_b(c,m)]|_{c^{*}}\mathbb{P}_{n}\left(\omega(a,x)-1\right)+\op(1)\\
    &=\sqrt{n}\mathrm{E}[\nabla_{c}u_b(c,m)]|_{c^{*}}\mathbb{P}_{n}\left(\omega(a,x)-1\right)+\op(1).
\end{align*}
Therefore, 
\begin{align*}
    \sqrt{n}(u_b(\hat{c},\hat{m})-\beta^{*})&=\mathbb{G}_{n}\left[u_b(1,m^{*})+\mathrm{E}[\nabla_{c}u_b(c,m)]|_{c^{*}}\left(\omega(a,x)-1\right)\right]+\op(1).
\end{align*}
From the form of the influence function \eqref{eq:inf}, this implies that the asymptotic MSE is 
\begin{align*}
   n^{-1}\min_{\zeta \in \mathbb{R}^{2},\tau \in \Theta_\tau}V_{\operatorname{snd}}(m(x,a;\zeta,\tau)). 
\end{align*}
\end{proof}

\begin{proof}[Proof of Theorem \ref{thm:drs1}]

As in the same way of Theorem \ref{thm:drs}, it is proved that the asymptotic MSE of $\hat{\beta}_{\mathrm{reg}}^{0}$ is 
\begin{align*}
    n^{-1}\min_{\zeta \in \mathbb{R}^{2},\tau \in \Theta_\tau}\mathrm{var}[v(\{\zeta_1+\zeta_2 q(x,a;\tau)\}_{t=0}^{T-1})]. 
\end{align*}
The asymptotic MSE of $\hat{\beta}_{\mathrm{sis}}$, $\hat{\beta}_{\mathrm{sn2reg}}$ and $\hat{\beta}_{\mathrm{dr}}$ is represented as a form of $\mathrm{var}[v(\{\zeta_1+\zeta_2 q(x,a;\tau)\}_{t=0}^{T-1})]$. When $\zeta=(0,0)$, it corresponds to the $\hat{\beta}_{\mathrm{sis}}$. When $\zeta=(\beta^{*},0)$, it corresponds to the $\hat{\beta}_{\mathrm{sn2sis}}$. When $\zeta=(0,1)$, it corresponds to $\hat{\beta}_{\mathrm{sndr}}$. 
\end{proof}

\section{Details of the experimental setup}
\label{sec:experiment-ape}

\subsection{Contextual bandit}

\textbf{Transformation method} A multi-label classification data set comprises $(x^{(i)},y^{(i)})_{i=1}^{n}$ where $x^{(i)}$ is covaraite and $y^{(i)}$ is its class. Here, $x^{(i)} \in \mathbb{R}^{d}$ and $y^{(i)} \in \{1,\cdots,l\}$, where $l$ is the number of class. A classification algorithm assigning $x$ to $y$ is considered to be a policy from a context to an action. 

Next, we will explain how to define a reward. The policy is considered to be an estimator $a^{(i)}$ associated with $x^{(i)}$. The agent receives a unit reward $1$ if the prediction succeeds, that is, when $a^{(i)}=y^{(i)}$. It receives no reward when $a^{(i)} \neq y^{(i)}$. The reward of a policy is considered to be the accuracy of the classification model. In this way, we can generate triplets of $\{(x^{(i)},a^{(i)},r^{(i)})\}_{i=1}^{n}$. In section \ref{sec:experiment}, based on some classification data set and some randomized policies, we made a data set $200$ times and performed simulations.

\textbf{Additional remarks} 
\begin{itemize}
    \item The data set is split into training data ($30\%$) for defining a policy $\pi_d$ and evaluation data ($70\%$) for the OPE. The size of the evaluation data is larger than that of training data because for the current problem, the accuracy of $\pi_d$ is not important and, we intend to know the accuracy of the OPE methods. 
    \item Our survey has denoted that several methods can be employed to construct Q-functions. For example, \cite{Chow2018} used a training data set to learn a Q-function. However, we should not use the training data for the valid comparison of OPE methods. In our case, the behavior policy was only applied to the evaluation data set. We subsequently constructed a Q-function using the generated data. 
    \item The number of actions and data points of the problem is shown in Table \ref{tab:dataset}. 
\end{itemize}
\begin{table}[]
    \centering
     \caption{Bandit Datasets}
    \begin{tabular}{ccccc}
         Dataset & PageBlock & OptDigits &SatImage & PenDigits   \\
         Classes &  5  & 10  & 6  & 10  \\
         Data &  5473 &  5620 &  6435 &  10992  
    \end{tabular}
    \label{tab:dataset}
\end{table}

\subsection{Reinforcement learning}

We hereby describe the RL domains used in the experiments.  

\textbf{Windy Gridworld}

A detailed explanation is Example 6.5 in \cite{SuttonRichardS1998Rl:a}. The board is a $7 \times 10$ matrix. The reward is $-1$ for all tranistion until the terminal state is reached.  The action comprises four choices: up, down, right, left. The difference of the usual GridWorld is that a crosswind runs upward through the middle of the grid. The horizon was set to $T=400$. Further, we calculated the best policy $\pi_d$ using Q-learning. 

\textbf{Cliff Walking}

The detailed explanation is Example  6.6 in \cite{SuttonRichardS1998Rl:a}. The board is a $4 \times 12$ matrix. Each time step incurs $-1$ reward, and stepping into the cliff incurs $-100$ reward and a reset to the start. An episode is terminated when the agent reaches the goal.
The horizon was set to $T=400$. Further, we calculated the best policy $\pi_d$ using Q-learning.

\textbf{Mountain Car}

A car is between two hills in interval $[-0.7,0.5]$ and the agent should move back and forth to gain enough power to reach the top of the right hill. The state space comprises position and velocity. There are three discrete actions 1)forward, 2)backward and 3) stay-still. The horizon was set to be $T=250$ with a reward of $-1$ per step. We calculated the best policy $\pi_e$ using Q-learning. The state space was continuous; thus, we obtained a $400$-dimensional feature using a radial basis function kernel.

\end{document}